\documentclass{article}

\PassOptionsToPackage{numbers,compress}{natbib}

\usepackage[preprint]{format}

\usepackage{verbatim}

\usepackage{comment}

\usepackage[utf8]{inputenc} %
\usepackage[T1]{fontenc}    %

\usepackage[colorlinks=true,
            citecolor=blue,
            linkcolor=red,
            anchorcolor=green]{hyperref}

\usepackage{url}            %
\usepackage{booktabs}       %
\usepackage{amsfonts}       %
\usepackage{nicefrac}       %
\usepackage{microtype}      %
\usepackage{xcolor}         %

\usepackage{xcolor}

\usepackage{booktabs}

\usepackage{comment}
\usepackage{booktabs}
\usepackage{multirow}
\usepackage{adjustbox}
\usepackage[table]{xcolor}
\usepackage{colortbl}
\usepackage{url}
\usepackage{xspace}
\usepackage{xcolor}
\usepackage{colortbl}

\usepackage[ruled,vlined]{algorithm2e}

\usepackage{multirow}
\usepackage{balance}
\usepackage{subfigure}
\usepackage{graphicx}
\usepackage{mdframed}
\usepackage{mathtools}
\usepackage{enumitem}
\usepackage[noabbrev,capitalise]{cleveref}

\usepackage{booktabs}
\usepackage{caption}   %
\usepackage{float}    %

\crefname{appendix}{Appendix}{appendices}
\Crefname{appendix}{Appendix}{Appendices}
\crefname{equation}{Eq.}{Eqs.}

\Crefname{equation}{Eq.}{Eqs.}

\usepackage{amsthm} 
\usepackage{amsthm}
\newtheorem{proposition}{Proposition}
\usepackage{amsmath}
\usepackage{amsthm}
\usepackage{bm}
\usepackage{ragged2e}
\usepackage{colortbl}
\usepackage{booktabs}
\usepackage{colortbl}
\usepackage{xcolor}
\usepackage{amsmath, amssymb}
\usepackage{tcolorbox}
\usepackage{pifont}  %

\usepackage{xcolor} 
\usepackage{booktabs}
\usepackage{xcolor, booktabs, tcolorbox, amsmath, enumitem}
\tcbuselibrary{skins, breakable}

\usepackage{xcolor, booktabs, tcolorbox, amsmath, enumitem}
\tcbuselibrary{skins, breakable}

\usepackage{xcolor}
\usepackage{setspace}
\usepackage{listings}

\usepackage{fontawesome5}

\definecolor{linkpink}{HTML}{E83E8C}

\newcommand{\ourmodel}{SOD\xspace}

\newcommand{\ie}{\emph{i.e.,}\xspace}
\newcommand{\eg}{\emph{e.g.,}\xspace}

\title{\ourmodel: Step-wise On-policy Distillation for \\
 Small Language Model Agents}

\author{%
\begin{minipage}{\textwidth}
\centering
\textbf{Qiyong Zhong}$^{1}$\thanks{Equal contribution.}  \ \thanks{This work was done during an internship at Tencent.} \quad
\textbf{Mao Zheng}$^{2}$\footnotemark[1] \quad
\textbf{Mingyang Song}$^{2}$\footnotemark[1] \quad
\textbf{Xin Lin}$^{1}$\footnotemark[1] \quad
\textbf{Jie Sun}$^{3}$ \quad \\[0.35em]
\textbf{Yiqi Zhao}$^{2}$ \quad
\textbf{Xiaodi Wang}$^{2}$ \quad
\textbf{Houcheng Jiang}$^{3}$\thanks{Corresponding authors.} \quad
\textbf{Xiang Wang}$^{3}$ \quad
\textbf{Junfeng Fang}$^{4}$\footnotemark[3] \\
\vspace{0.8em}
{\normalfont
$^{1}$Zhejiang University \quad
$^{2}$Large Language Model Department, Tencent \\
$^{3}$University of Science and Technology of China \quad
$^{4}$National University of Singapore \\
}
\vspace{0.8em}
{\normalfont\small
\texttt{\{youngzhong,linxin2\}@zju.edu.cn} ;
\texttt{\{moonzheng,nickmysong\}@tencent.com} \\
\texttt{\{sunjie2019,jianghc\}@mail.ustc.edu.cn} ;
\texttt{\{yiqizhao02,wxd13305645488,xiangwang1223\}@gmail.com; fangjf@nus.edu.sg}
\\[0.5em]
\href{https://github.com/YoungZ365/SOD}{\textcolor{black}{\faGithub}\;\textcolor{linkpink}{\texttt{github.com/YoungZ365/SOD}}}
\\[0.25em]
\href{https://huggingface.co/collections/youngzhong/sod}{\raisebox{-0.18em}{\includegraphics[height=1.15em]{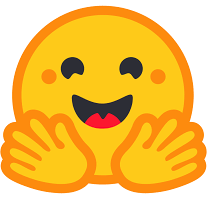}}\;\textcolor{linkpink}{\texttt{huggingface.co/collections/youngzhong/sod}}}
}
\end{minipage}
}

\usepackage{fancyhdr}
\usepackage{graphicx}
\usepackage{datetime}  %

\fancypagestyle{firstpagestyle}{%
  \fancyhf{}  
  \fancyhead[L]{\includegraphics[height=0.66cm]{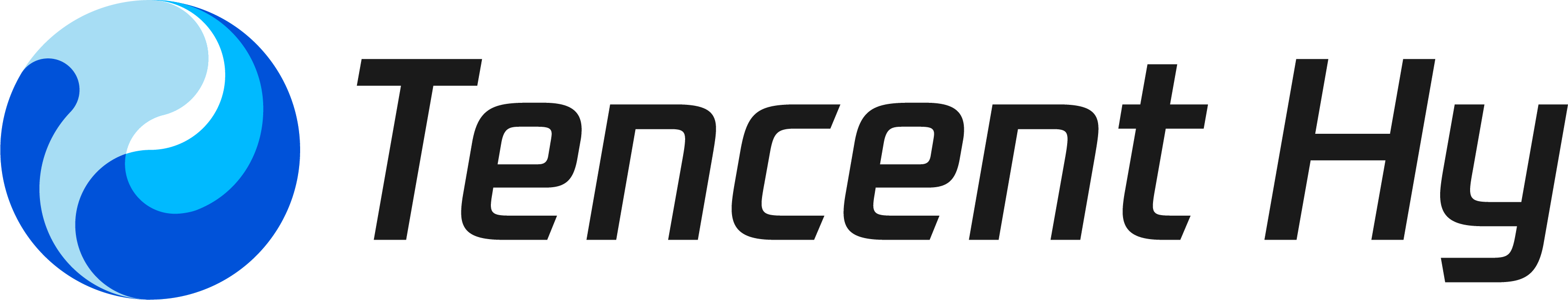}}
  \fancyhead[R]{\textcolor{gray}{\fontsize{11pt}{10pt}\selectfont May 8, 2026}}

}

\begin{document}

\maketitle
\thispagestyle{firstpagestyle} 
\vspace{-1.95em}
\begin{abstract}
Tool-integrated reasoning (TIR) is difficult to scale to small language models due to instability in long-horizon tool interactions and limited model capacity. While reinforcement learning methods like group relative policy optimization provide only sparse outcome-level rewards. Recently, on-policy distillation (OPD) has gained popularity by supplying dense token-level supervision from a teacher on student-generated trajectories. However, our experiments indicate that applying OPD to TIR leads to a critical failure mode: erroneous tool calls tend to cascade across subsequent reasoning steps, progressively amplifying student-teacher divergence and rendering the teacher's token-level supervision increasingly unreliable.
To address this, we propose \ourmodel, a step-wise on-policy distillation framework for small language model agents, which adaptively reweights distillation strength at each step based on step-level divergence. Therefore, \ourmodel can attenuate potentially misleading teacher signals in high-divergence regions while preserving dense guidance in well-aligned states. Experiments on challenging math, science, and code benchmarks show that \ourmodel achieves up to \textbf{20.86\%} improvement over the second-best baseline. \textbf{Notably, our 0.6B student achieves 26.13\% on average@32 at AIME 2025}, demonstrating effective transfer of agentic reasoning to small models.

\end{abstract}

\section{Introduction}\label{sec:introduction}

\begin{figure*}[t]
\centering
\includegraphics[width=1.0\linewidth]{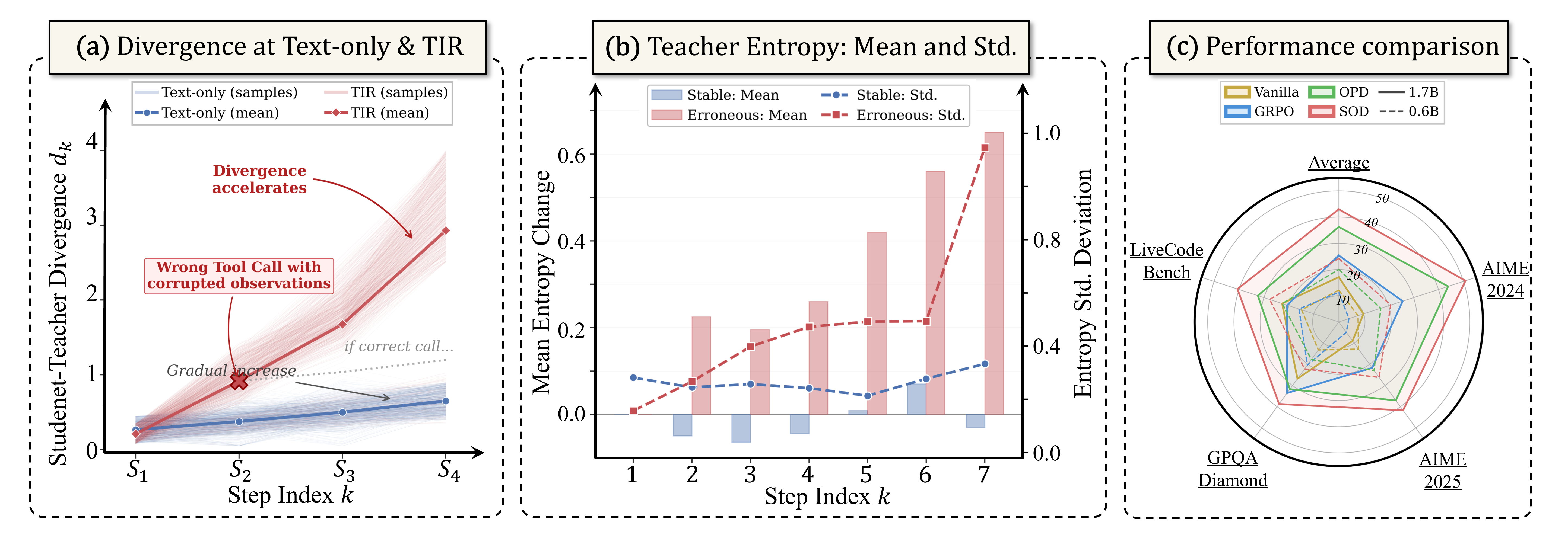}
\caption{\textbf{The motivation of \ourmodel}. 
(a) Student-teacher divergence $d_k$ across reasoning steps, sampled from 800 trajectories: in TIR, erroneous tool calls cause divergence to accelerate sharply, unlike the gradual drift in text-only reasoning. 
(b) Teacher entropy statistics over 800 sampled trajectories: on erroneous trajectories, both the mean entropy change (bars) and the standard deviation (dashed lines) grow rapidly at later steps, indicating increasingly unreliable teacher supervision. Please refer to~\Cref{app:H} for detailed cases.
(c) Radar chart comparing methods on four benchmarks.}
\label{motivation}
\vspace{-2 em}
\end{figure*}

Agentic capabilities have substantially expanded the applicability of large language models (LLMs)~\citep{xi2023rise,AgentDistillation}, enabling them to solve complex real-world tasks~\citep{react,schick2023toolformer,singh2025agentic}. 
However, such capabilities are typically realized by large-scale models, which incurs substantial inference cost, deployment overhead, and system complexity~\citep{singh2025agentic,chenglin2024mixed,fan2026agentprocessbench}. In latency-, resource-, and privacy-sensitive scenarios, transferring agentic capabilities to small language models (SLMs) that can be deployed on-device is therefore of significant practical importance~\citep{xu2024ondevice,xu2024kd_llm_survey,Agentdistill,oresearcher,coa}. Despite this motivation, enabling SLMs to acquire stable and effective tool-integrated reasoning (TIR) abilities remains a major challenge~\citep{qian2025toolrl,rainone2025replacing,xue2025simpletir,sad}.

Existing post-training methods for enhancing TIR abilities are largely based on reinforcement learning (RL)~\citep{li2025torl,jin2025searchr1,song2025r1searcher}, particularly policy optimization algorithms such as group relative policy optimization (GRPO)~\citep{grpo}. However, in SLM-based TIR settings, RL often suffers from unstable optimization~\citep{qian2025toolrl,rainone2025replacing}. TIR tasks typically involve long-horizon trajectories~\citep{li2025torl}, multi-step decision making~\citep{jin2025searchr1}, and interactions with external tools~\citep{singh2025agentic}, whereas RL commonly provides only sparse outcome-level rewards~\citep{deepseek2025r1}. For small models with limited capacity and weaker exploration ability, such sparse supervision can further exacerbate exploration failure, leaving the policy in a cold-start regime with few informative reasoning signals~\citep{kepo}.

Recently, on-policy distillation (OPD) has emerged as a promising paradigm for post-training~\citep{opd,gu2024minillm,eopd,ropd,lightOPD}. Unlike RL methods that rely on sparse, trajectory-level rewards~\citep{deepseek2025r1}, OPD provides dense token-level supervision~\citep{gu2024minillm} on trajectories sampled from the student's own policy~\citep{opd}, thereby alleviating the credit assignment difficulty inherent in sparse reward signals~\citep{qwen2025qwen3,li2026rethinking} while substantially improving sample efficiency~\citep{xu2026rlkd} and training stability~\citep{qwen2025qwen3}.

However, our experiments show that directly transferring OPD to SLM-based TIR can lead to severe training instability~\citep{srpo,sdrlvr,bousselham2025vold,wang2026openclaw}. We attribute this failure to a fundamental difference between how student trajectories deviate from the teacher distribution in TIR and in standard text-based reasoning~\citep{li2026rethinking,fu2026revisiting,opdsurvey,bousselham2025vold}. As illustrated in~\cref{motivation}(a), in text-only reasoning, even when the student gradually departs from the teacher, subsequent states still evolve continuously along the generated context, and the resulting distribution shift is typically progressive~\citep{ross2011reduction,opd}. In contrast, TIR introduces discontinuous state transitions through tool interactions~\citep{li2025torl}. A single erroneous tool call can inject an incorrect observation into the context, causing subsequent reasoning steps to unfold from a corrupted state~\citep{qian2025toolrl}. This issue is further amplified because teacher models are often highly accurate in tool use, whereas small models have limited capacity and weaker exploration ability~\citep{gudibande2024false,rainone2025replacing}. During early training, small models may accumulate multiple erroneous tool calls, causing their state distribution to rapidly drift away from the teacher distribution~\citep{li2026rethinking,wang2026tcod}. In such out-of-distribution states, teacher-provided token-level supervision may become unreliable or even misleading~\citep{veto,bousselham2025vold,kepo,tip,srpo} as illustrated in~\cref{motivation}(b), and the enlarged student-teacher discrepancy can induce unstable gradients and even training collapse~\citep{li2026rethinking,qwen2025qwen3}.

Based on this observation, we propose \textbf{\ourmodel}, a 
\textbf{S}tep-wise \textbf{O}n-policy \textbf{D}istillation framework for small language model agents. The core idea is to adaptively adjust the distillation strength at each reasoning step according to the divergence between the student and teacher. Specifically, when the student remains well aligned with the teacher, \ourmodel preserves dense supervision to fully exploit teacher guidance; when the deviation becomes large, it progressively attenuates the distillation signal for the corresponding step, thereby reducing the influence of potentially misleading supervision under out-of-distribution states. 
Experimental results (as displayed in~\cref{motivation}(c)) demonstrate that this simple yet effective step-wise reweighting mechanism enables lightweight agents to acquire agentic TIR capabilities more efficiently. 
Extensive experiments on challenging math, science, and code benchmarks show that \ourmodel outperforms the second-best baseline by up to \textbf{20.86\%}.
Notably, our 0.6B-scale student achieves \textbf{26.13\%} accuracy on AIME 2025. To our best knowledge, this is the first sub-billion-parameter model to reach this level on such a challenging reasoning benchmark.

\section{Related Work}\label{sec:related}

\paragraph{Reinforcement Learning for Agents.}
RL-based post-training has evolved from reinforcement learning from human feedback (RLHF)~\citep{christiano2017deep,ziegler2019fine} with PPO~\citep{schulman2017proximal} to more scalable methods like GRPO~\citep{grpo}. For language agents, structured reasoning paradigms such as ReAct~\citep{react}, Toolformer~\citep{schick2023toolformer}, and FireAct~\citep{chen2023fireact} enable tool use but rely on demonstrations rather than online optimization. Recent work extends RL to agent interaction trajectories across code generation~\citep{yang2024swe}, tool use~\citep{feng2025retool}, GUI interaction~\citep{bai2024digirl}, and web navigation~\citep{qi2024webrl}. A central challenge is credit assignment under sparse, delayed feedback, addressed via trajectory-level updates and value-free formulations~\citep{zhou2024archer, chen2025reinforcement}. KL-regularized policy optimization further introduces bias and instability concerns~\citep{schulman2017proximal}, amplified in agentic settings by distribution shift and compounding errors. Broader frameworks scaling agentic RL across environments~\citep{demy, wang2026rlanything, wang2025co} still rely on trajectory-level rewards without dense supervision.

\paragraph{On-policy Distillation.}
On-Policy Distillation (OPD)~\citep{opdsurvey,opcd,gad,hybridOPD, scope,tip,chen2026soda} introduces token-level supervision on student-generated trajectories to mitigate the distribution mismatch of offline distillation. \citet{gu2024minillm} formulates OPD as reverse KL minimization under the student distribution, while \citet{opd} unifies on-policy and off-policy distillation across divergence objectives. \citet{yang2026learning} further interprets OPD as KL-regularized reinforcement learning with implicit per-token rewards, and \citet{li2026rethinking} shows that OPD primarily aligns local support on student-visited states, depending on teacher–student compatibility in reasoning patterns. OPD has also been extended to self-distillation settings without external teachers~\citep{opsd,opsd2,sdpo,sdft,dqchen}.
Recently, more works~\citep{srpo,sdrlvr,bousselham2025vold,wang2026openclaw,skillsd} have explored introducing OPD to mitigate the limitations of RL.

\section{Preliminaries}

\subsection{Multi-turn Tool-Integrated Reasoning}

We consider the post-training of a small language model (SLM) for multi-turn tool-integrated reasoning (TIR). Given an input $x$, the model interacts with an external environment over multiple reasoning steps. At each step $k$, the model generates a response $y_k$, which may include natural-language reasoning, a tool invocation, or a final answer. If a tool is invoked, the environment returns an observation $o_k$, which is appended to the context and conditions subsequent generations.

A trajectory is defined as:
\begin{equation}
    \tau = (x, y_1, o_1, \ldots, y_K, o_K, y_{K+1}),
\end{equation}
where $y_{K+1}$ denotes the final response. The policy $\pi_\theta$ generates only model tokens, while observations $\{o_k\}$ are provided by the environment. Let $y_t$ denote a generated token and $y_{<t}$ its prefix, which may include both model outputs and tool observations.

\subsection{Group Relative Policy Optimization}

Group Relative Policy Optimization (GRPO)~\citep{grpo} is a reinforcement learning algorithm that updates the policy using relative rewards within a group of sampled trajectories. We assume access to an outcome-level reward function $r(\tau)$ defined on complete trajectories. For each input $x$, a group of trajectories $\{\tau_i\}_{i=1}^G$ is sampled from the old policy $\pi_{\theta_{\mathrm{old}}}$, each receiving reward $r_i = r(\tau_i)$.

The group-relative advantage is computed as:
\begin{equation}
    \hat{A}_i =
    \frac{r_i - \mathrm{mean}(\{r_j\})}
    {\mathrm{std}(\{r_j\}) + \epsilon_A}.
\label{eq:advantage}
\end{equation}

Let $\rho_{i,t}(\theta) = \frac{\pi_\theta(y_{i,t} \mid y_{i,<t})}{\pi_{\theta_{\mathrm{old}}}(y_{i,t} \mid y_{i,<t})}$ denote the token-level importance ratio. The objective is defined as:
\begin{equation}
\mathcal{L}_{\mathrm{GRPO}}
=
-
\mathbb{E}
\left[
\frac{1}{G}
\sum_{i=1}^{G}
\frac{1}{|\mathcal{T}_i|}
\sum_{t \in \mathcal{T}_i}
\min
\left(
\rho_{i,t}(\theta)\hat{A}_i,
\mathrm{clip}(\rho_{i,t}(\theta), 1-\epsilon, 1+\epsilon)\hat{A}_i
\right)
\right],
\label{eq:grpo}
\end{equation}
where $\mathcal{T}_i$ denotes model-generated token positions.
This objective provides an on-policy learning signal based on relative trajectory performance, but supplies only sparse outcome-level rewards.

\subsection{On-policy Distillation}

On-policy distillation (OPD)~\citep{opd,lu2025onpolicydistillation} is a post-training paradigm that provides dense token-level supervision on student-generated trajectories by aligning the student policy with a teacher distribution. Given a trajectory $y \sim \pi_\theta$, the OPD objective is defined as:
\begin{equation}
\mathcal{L}_{\mathrm{OPD}}^{\mathrm{}}
=
\mathbb{E}
\left[
\sum_{t \in \mathcal{T}_i}
\left(
\log \pi_\theta(y_t \mid y_{<t})
-
\log \pi_{\mathrm{teacher}}(y_t \mid y_{<t})
\right)
\right],
\end{equation}

where $\pi_\theta$ is the student policy, $\pi_{\mathrm{teacher}}$ is the teacher model, and $\mathcal{T}_i$ denotes the set of model-generated token positions in the sampled trajectory. This objective corresponds to a sampled estimator of the reverse KL divergence between the student and teacher policies on student-visited states.

\section{Methodology}

In this section, we introduce \ourmodel.
The framework of \ourmodel is displayed in~\cref{framework}.
We first analyze why on-policy distillation (OPD) can become unreliable under tool-induced state drift. We then introduce a step-level divergence score as an observable proxy for the reliability of teacher supervision, and finally formulate a step-wise reweighted OPD objective.

\subsection{Failure of On-policy Distillation under Tool-Induced State Drift}
\label{sec:failure_uniform_opd}

The OPD objective assumes that teacher supervision remains reliable across all student-visited states~\citep{fu2026revisiting,li2026rethinking,bousselham2025vold}. This assumption holds when distribution drift is gradual, but is severely violated in TIR due to discontinuous state transitions introduced by tool observations. We formalize this failure through two propositions (proofs in~\cref{app:proof_opd_failure}).

We define the step-level mismatch as:
\begin{equation}
    \Delta_k = \frac{1}{|\mathcal{I}_k|} \sum_{t\in \mathcal{I}_k} D_{\mathrm{KL}}\bigl(\pi_\theta(\cdot \mid y_{<t}) \,\big\Vert\, \pi_{\mathrm{teacher}}(\cdot \mid y_{<t})\bigr).
\end{equation}

\begin{proposition}[Discontinuous divergence amplification]
\label{prop:discontinuous_drift}
In text-only reasoning, $\Delta_{k+1} - \Delta_k = O(\eta)$ where $\eta$ is the per-token distributional shift. In TIR, a single erroneous tool observation of length $m$ causes $\Delta_{k+1} - \Delta_k = \Omega(m \cdot \eta_{\mathrm{tool}})$ with $\eta_{\mathrm{tool}} \gg \eta$. Under $j$ consecutive tool errors, the divergence compounds super-linearly: $\Delta_{k+j} - \Delta_k = \Omega\!\left(\sum_{i=0}^{j-1} m_i \cdot \eta_{\mathrm{tool}}^{(i)}\right)$ with $\eta_{\mathrm{tool}}^{(i+1)} \ge \eta_{\mathrm{tool}}^{(i)}$.
\end{proposition}

\begin{proposition}[Gradient SNR degradation]
\label{prop:variance_explosion}
Define the teacher-supported region $S_t^\epsilon = \{v \in \mathcal{V}: \pi_{\mathrm{teacher}}(v \mid y_{<t}) \ge \epsilon\}$ and the overlap $\rho_t = \sum_{v \in S_t^\epsilon} \pi_\theta(v \mid y_{<t})$. When $\rho_t \le \rho$, the second moment of the OPD loss satisfies $\mathbb{E}[\ell_t^2] \ge (1 - \rho)\,\log^2(1/\epsilon)$, and the signal-to-noise ratio of the gradient estimator $\mathrm{SNR}(g_t) \to 0$ as $\rho_t \to 0$.
\end{proposition}

Together, these results characterize a failure cascade specific to TIR: consecutive erroneous tool calls trigger compounding divergence amplification (Prop.~\ref{prop:discontinuous_drift}), progressively pushing the student away from the teacher distribution into low-overlap states where the OPD gradient becomes dominated by high-variance, uninformative contributions (Prop.~\ref{prop:variance_explosion}). This is empirically confirmed in~\cref{motivation}: divergence accelerates sharply as tool errors accumulate (a), and teacher entropy becomes elevated and unstable in subsequent steps (b). Since OPD aggregates losses uniformly across steps, it systematically overweights these corrupted signals, motivating our step-wise reweighting mechanism.
We formally show in Appendix~\ref{app:variance_reduction} that under monotonically increasing divergence, our reweighting suppresses the weighted second moment by a factor of $O((d_1/d_k)^2)$, restoring bounded gradient SNR even in low-overlap states where OPD suffers SNR collapse.

\begin{figure*}[t]
\centering
\includegraphics[width=1.0\linewidth]{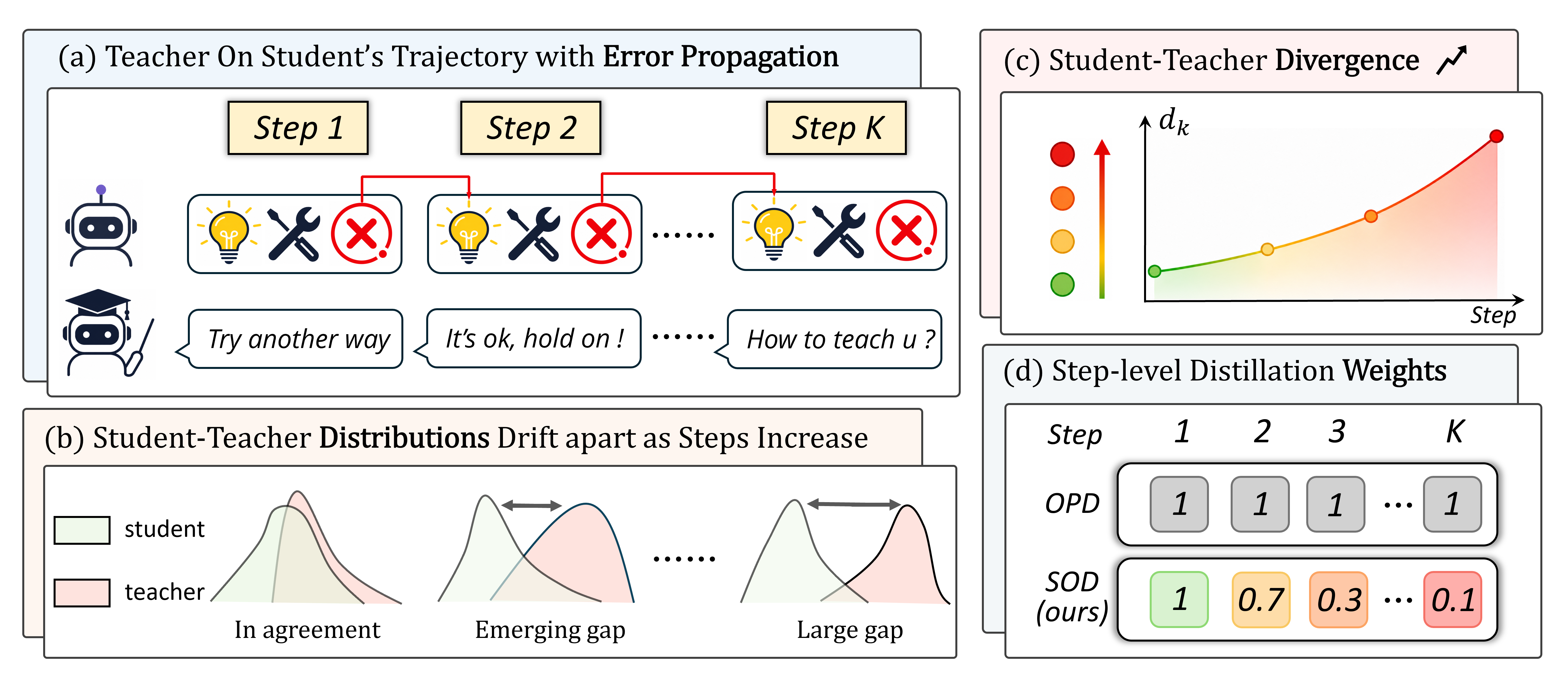}
\caption{\textbf{The overview of \ourmodel}. (a)~The student generates multi-step trajectories where erroneous tool calls propagate across steps, degrading teacher supervision reliability. (b)~Student-teacher distributions drift apart as errors accumulate. (c)~Step-level divergence $d_k$ quantifies this drift. (d)~\ourmodel adaptively attenuates distillation weights in high-divergence steps, unlike vanilla OPD which applies uniform weights. Please refer to~\cref{rq5_body} for more characteristics of \ourmodel.}
\label{framework}
 \vspace{-0.09 in}
\end{figure*}

\subsection{Student-teacher Divergence and Step-wise Reweighting}

We partition a generated trajectory into $K+1$ reasoning steps, where each step corresponds to a model response between two tool observations, or the final answer step. Let $\mathcal{I}_k$ denote the set of token positions belonging to the $k$-th model-generated step. Tool observation tokens are excluded because they are produced by the external environment rather than by the policy.

We define a step-level divergence score to quantify student–teacher divergence at step $k$:

\begin{equation}
    d_k =
    \frac{1}{|\mathcal{I}_k|}
    \sum_{t \in \mathcal{I}_k}
    \left|
    \log \pi_\theta(y_t \mid y_{<t})
    -
    \log \pi_{\mathrm{teacher}}(y_t \mid y_{<t})
    \right|.
\label{eq:step_div}
\end{equation}

The score $d_k$ serves as a lightweight indicator of the local reliability of teacher supervision. A small $d_k$ suggests that the student remains well aligned with the teacher distribution, while a large $d_k$ indicates substantial mismatch, often caused by tool-induced state drift or corrupted observations.

Based on these divergence scores, we adaptively reweight the strength of distillation across steps.
We initialize the first step with full distillation strength, \ie $w_1 = 1$. For subsequent steps, the weight is:
\begin{equation}
    w_k =
    \min\!\left(
    w_1
    \prod_{u=1}^{k-1}
    \frac{d_u + \epsilon}{d_{u+1} + \epsilon},
    1 + \delta
    \right),
    \quad k \ge 2,
\label{eq:adaptive_weight}
\end{equation}
where $\epsilon$ is a small constant for stability and $\delta$ controls the maximum amplification of the distillation.
Since the reweighting in~\cref{eq:adaptive_weight} depends only on the \emph{ratios} between consecutive divergence scores rather than their absolute values, any monotone proxy of $\Delta_k$ suffices. We show in~\cref{app:dk_proxy} that $d_k$ is monotonically consistent with $\Delta_k$ and computable at zero marginal cost from the OPD forward pass. We therefore adopt $d_k$ in place of $\Delta_k$ in our implementation.

This weighting rule captures the evolution of student-teacher divergence along the trajectory. When the divergence increases, \ie $d_{u+1} > d_u$, indicating that the student is drifting away from the teacher distribution (often due to tool-induced state shift), the corresponding ratio becomes smaller than one, leading to attenuation of the distillation signal in high-mismatch regions. 
Conversely, the student may partially move back toward the teacher-supported region in later steps, effectively correcting earlier deviations. We refer to this phenomenon as \textit{recovery from earlier errors}. In such cases, when $d_{u+1} < d_u$, the reweighting mechanism allows the distillation strength to increase accordingly, restoring informative guidance. To ensure stable optimization, the weight is upper-bounded by $1+\delta$.

\subsection{Training Objective}

For a student-generated trajectory, the sampled-token OPD term at token $t$ is:
\begin{equation}
    \ell_{\mathrm{OPD}}(y_t)
    =
    \log \pi_\theta(y_t \mid y_{<t})
    -
    \log \pi_{\mathrm{teacher}}(y_t \mid y_{<t}).
\end{equation}

Instead of applying this term uniformly at step level, we reweight all tokens in step $k$ by the corresponding reliability weight $w_k$, leading to the step-wise OPD objective:
\begin{equation}
    \mathcal{L}_{\mathrm{OPD}}^{\mathrm{step}}
    =
    \mathbb{E}_{y \sim \pi_\theta}
    \left[
    \sum_{k=1}^{K+1}
    w_k
    \sum_{t \in \mathcal{I}_k}
    \left(
    \log \pi_\theta(y_t \mid y_{<t})
    -
    \log \pi_{\mathrm{teacher}}(y_t \mid y_{<t})
    \right)
    \right].
\label{eq:step_opd}
\end{equation}

The final training objective is defined as:
\begin{equation}
\label{eq:training_object}
    \mathcal{L}
    =
    \mathcal{L}_{\mathrm{GRPO}}
    +
    \mathcal{L}_{\mathrm{OPD}}^{\mathrm{step}}.
\end{equation}

Here, $\mathcal{L}_{\mathrm{GRPO}}$ provides sparse outcome-level rewards to drive trajectory exploration, while $\mathcal{L}_{\mathrm{OPD}}^{\mathrm{step}}$ supplies dense token-level guidance whose strength is reweighted by the student-teacher divergence, jointly enabling stable learning under tool-induced state drift.

\section{Experiment}\label{sec:experiment}

\subsection{Experimental Setup}

\paragraph{Datasets \& Benchmarks.} Following~\citet{demy}, our training data comprises a 3k high-quality SFT corpus of multi-turn reasoning trajectories curated from s1-1k~\citep{muennighoff2025s1}, LeetCode, and ReTool~\citep{feng2025retool}, along with a 30k diverse RL dataset covering mathematical reasoning from DAPO-Math~\citep{yu2025dapo}, math and coding tasks from Skywork-or1~\citep{he2025skywork-or1}, and scientific problem solving from MegaScience~\citep{fan2025megascience}. Please refer to~\Cref{app:datasets} for more details. We evaluate on 4 challenging benchmarks: AIME 2024/2025, GPQA-Diamond~\citep{rein2024gpqa}, and LiveCodeBench~\citep{jainlivecodebench} (See~\cref{app:benchmarks} for details).

\paragraph{Evaluation Setups.}

The temperature is fixed at 1.0 with nucleus sampling parameter top\_p=0.6. For each problem, 32 independent samples are generated to enable a more thorough evaluation, from which average@32 are computed and finally reported (percentage). %
The teacher model is a Qwen3-4B model~\citep{qwen2025qwen3} further optimized with GRPO on the RL datasets.
If not otherwise specified, the teacher of subsequent experiments is the 4B model. %
We use Qwen3-0.6B and Qwen3-1.7B as student models.

\paragraph{Baselines \& Implementation Details.}
For the baselines,
we evaluate \ourmodel against a set of supervised, RL, and distillation baselines.
For each student, we compare \ourmodel with the following baselines:
(1) Vanilla.
(2) SFT.
(3) GRPO~\citep{grpo}.
(4) OPD~\citep{opd}.
(5) OPSD$_{\mathrm{gt}}$~\citep{opsd} (On-policy self-distillation with groundtruth).
(6) OPSD$_{\mathrm{hint}}$~\citep{opsd} (On-policy self-distillation with hints).
For more details about the baselines, please refer to~\Cref{app:baselines}.
For implementation details, please refer to~\Cref{app:implentation_details}.

\begin{table*}[t]
    \centering
    \caption{
    Performance comparison of Qwen3 teacher and student across 4 benchmarks. The best results are in \textbf{bold}, and the second-best results are \underline{underlined}. We report average@32 over 5 runs.}
    \vspace{1 mm}
    \label{tab:main}

    \begingroup
    \setlength{\tabcolsep}{6.5 pt}
    \renewcommand{\arraystretch}{1.08}
    \normalsize

    \begin{adjustbox}{max width=0.82\textwidth}
    \begin{tabular}{llccccc}
        \toprule
        \multirow{2}{*}{\textbf{Params}} 
        & \multirow{2}{*}{\textbf{Method}} 
        & \multicolumn{2}{c}{\textbf{Math}} 
        & \textbf{Science} 
        & \textbf{Code} 
        & \multirow{2}{*}{\textbf{Average}} \\
        \cmidrule(lr){3-4}
        \cmidrule(lr){5-5}
        \cmidrule(lr){6-6}
        & & AIME 2024 & AIME 2025 & GPQA & LiveCodeBench & \\

        \midrule
        \rowcolor[HTML]{FEE090}
        \multicolumn{7}{c}{\textbf{Teacher Models from the Qwen3 Series}}\\
        \textbf{4B} 
        & GRPO 
        & 67.60\phantom{.}\scalebox{0.72}{$\pm$1.34} & 60.42\phantom{.}\scalebox{0.72}{$\pm$1.47} & 55.19\phantom{.}\scalebox{0.72}{$\pm$0.81} & 63.13\phantom{.}\scalebox{0.72}{$\pm$0.93} & 61.59 \\

        \midrule
        \rowcolor[HTML]{E0F3F8}
        \multicolumn{7}{c}{\textbf{Distilled Student Models from the Qwen3 Series}}\\
        \multirow{7}{*}{\textbf{0.6B}}
        & Vanilla 
        & 7.71\phantom{.}\scalebox{0.72}{$\pm$0.33} & 12.81\phantom{.}\scalebox{0.72}{$\pm$0.38} & 13.24\phantom{.}\scalebox{0.72}{$\pm$0.24} & 14.89\phantom{.}\scalebox{0.72}{$\pm$0.29} & 12.16 \\
        & SFT 
        & 5.67\phantom{.}\scalebox{0.72}{$\pm$0.26} & 5.42\phantom{.}\scalebox{0.72}{$\pm$0.22} & 15.20\phantom{.}\scalebox{0.72}{$\pm$0.30} & 9.61\phantom{.}\scalebox{0.72}{$\pm$0.27} & 8.97 \\
        & GRPO 
        & 4.06\phantom{.}\scalebox{0.72}{$\pm$0.37} & 4.90\phantom{.}\scalebox{0.72}{$\pm$0.31} & \underline{20.38}\phantom{.}\scalebox{0.72}{$\pm$0.47} & 15.95\phantom{.}\scalebox{0.72}{$\pm$0.51} & 11.32 \\
        & OPD 
        & \underline{16.82}\phantom{.}\scalebox{0.72}{$\pm$0.81} & \underline{22.95}\phantom{.}\scalebox{0.72}{$\pm$0.97} 
        & 17.76\phantom{.}\scalebox{0.72}{$\pm$0.41} 
        & \underline{22.65}\phantom{.}\scalebox{0.72}{$\pm$0.75} & \underline{20.04} \\
        & OPSD$_{\mathrm{gt}}$ 
        & 12.63\phantom{.}\scalebox{0.72}{$\pm$0.58} & 17.04\phantom{.}\scalebox{0.72}{$\pm$0.64} & 17.32\phantom{.}\scalebox{0.72}{$\pm$0.35} & 16.73\phantom{.}\scalebox{0.72}{$\pm$0.44} & 15.93 \\
        & OPSD$_{\mathrm{hint}}$ 
        & 9.77\phantom{.}\scalebox{0.72}{$\pm$0.43} & 14.12\phantom{.}\scalebox{0.72}{$\pm$0.48} & 15.98\phantom{.}\scalebox{0.72}{$\pm$0.30} & 12.65\phantom{.}\scalebox{0.72}{$\pm$0.37} & 13.13 \\
        & \textbf{\ourmodel} 
        & \textbf{20.84}\phantom{.}\scalebox{0.72}{$\pm$0.90} & \textbf{26.13}\phantom{.}\scalebox{0.72}{$\pm$1.07} & \textbf{22.19}\phantom{.}\scalebox{0.72}{$\pm$0.59} 
        & \textbf{27.72}\phantom{.}\scalebox{0.72}{$\pm$0.83} & \textbf{24.22} \\

        \midrule
        \multirow{7}{*}{\textbf{1.7B}}
        & Vanilla 
        & 9.90\phantom{.}\scalebox{0.72}{$\pm$0.38} & 8.96\phantom{.}\scalebox{0.72}{$\pm$0.33} & 26.80\phantom{.}\scalebox{0.72}{$\pm$0.36} & 22.73\phantom{.}\scalebox{0.72}{$\pm$0.41} & 17.10 \\
        & SFT 
        & 26.77\phantom{.}\scalebox{0.72}{$\pm$0.66} & 22.40\phantom{.}\scalebox{0.72}{$\pm$0.72} & 29.85\phantom{.}\scalebox{0.72}{$\pm$0.51} & 24.63\phantom{.}\scalebox{0.72}{$\pm$0.57} & 25.91 \\
        & GRPO 
        & 25.63\phantom{.}\scalebox{0.72}{$\pm$1.03} & 21.67\phantom{.}\scalebox{0.72}{$\pm$1.12} & 33.55\phantom{.}\scalebox{0.72}{$\pm$0.76} & 20.70\phantom{.}\scalebox{0.72}{$\pm$0.87} & 25.39 \\
        & OPD 
        & \underline{43.86}\phantom{.}\scalebox{0.72}{$\pm$1.23} & \underline{37.04}\phantom{.}\scalebox{0.72}{$\pm$1.31} 
        & 31.73\phantom{.}\scalebox{0.72}{$\pm$0.55} 
        & \underline{32.45}\phantom{.}\scalebox{0.72}{$\pm$0.93} 
        & \underline{36.27} \\
        & OPSD$_{\mathrm{gt}}$ 
        & 33.85\phantom{.}\scalebox{0.72}{$\pm$0.88} & 24.69\phantom{.}\scalebox{0.72}{$\pm$0.94} & \underline{35.02}\phantom{.}\scalebox{0.72}{$\pm$0.67} & 22.73\phantom{.}\scalebox{0.72}{$\pm$0.61} & 29.07 \\
        & OPSD$_{\mathrm{hint}}$ 
        & 34.42\phantom{.}\scalebox{0.72}{$\pm$0.76} & 21.43\phantom{.}\scalebox{0.72}{$\pm$0.83} & 33.46\phantom{.}\scalebox{0.72}{$\pm$0.59} & 23.12\phantom{.}\scalebox{0.72}{$\pm$0.52} & 28.11 \\
        & \textbf{\ourmodel} 
        & \textbf{50.83}\phantom{.}\scalebox{0.72}{$\pm$1.15} & \textbf{41.72}\phantom{.}\scalebox{0.72}{$\pm$1.24} & \textbf{38.72}\phantom{.}\scalebox{0.72}{$\pm$0.73} 
        & \textbf{40.63}\phantom{.}\scalebox{0.72}{$\pm$0.91} & \textbf{42.98} \\

        \bottomrule
    \end{tabular}%
    \end{adjustbox}

    \endgroup

    \vspace{-10pt}
\end{table*}

\subsection{Overall Performance}

We compare \ourmodel against six baselines on both 0.6B and 1.7B student models. The results are summarized in Table~\ref{tab:main}, from which we draw the following key observations. Note that we also provide training cost analysis in~\cref{app:cost}.

\begin{itemize}[leftmargin=*]

\item \textbf{Obs 1: \ourmodel consistently outperforms all baselines across every benchmark.}
On both model scales, \ourmodel achieves the highest scores on all four tasks, surpassing the strongest baseline OPD by \textbf{20.86\%} (0.6B) and \textbf{18.50\%} (1.7B) in relative average improvement. Notably, SFT and GRPO both fail to outperform the Vanilla baseline at the 0.6B scale, indicating that sparse outcome-level rewards and static demonstrations are insufficient for guiding small models in TIR. In contrast, our step-wise reweighting mechanism provides dense yet reliability-modulated supervision, effectively preventing misleading gradients from corrupted tool-call states while preserving informative teacher guidance in well-aligned regions.

\item \textbf{Obs 2: The improvements generalize across different model scales, demonstrating the scalability of \ourmodel.}
Our 1.7B student recovers \textbf{69.8\%} of the 4B teacher's performance, compared to only 58.9\% for OPD, showing that \ourmodel substantially improves distillation efficiency. Furthermore, even our 0.6B model surpasses several 1.7B baselines, suggesting that step-wise reweighting can partially compensate for the capacity gap between model scales. %

\end{itemize}

\begin{table*}[t]
    \centering
    \caption{
    \textbf{Ablation study on key components of \ourmodel.}
    Results are evaluated on 1.7B student model.}

    \vspace{1 mm}
    \label{tab:ablation}
    
    \begingroup
    \setlength{\tabcolsep}{4pt}
    \renewcommand{\arraystretch}{1.08}
    \normalsize

    \begin{adjustbox}{max width=0.90\textwidth}
    \begin{tabular}{lccccc}
        \toprule
        \multirow{2}{*}{\textbf{Variants}}
        & \multicolumn{2}{c}{\textbf{Math}}
        & \textbf{Science}
        & \textbf{Code}
        & \multirow{2}{*}{\textbf{Average}} \\
        \cmidrule(lr){2-3}
        \cmidrule(lr){4-4}
        \cmidrule(lr){5-5}
        & AIME 2024 & AIME 2025 & GPQA & LiveCodeBench & \\

        \midrule
        \rowcolor[HTML]{FEE090}
        \multicolumn{6}{c}{\textbf{Step-wise Reweighting}}\\
        (1.1) \textit{w/} Uniform Weighting
        & 41.68\phantom{.}\scalebox{0.72}{$\pm$0.95} & 35.58\phantom{.}\scalebox{0.72}{$\pm$1.01} & 30.12\phantom{.}\scalebox{0.72}{$\pm$0.59} & 31.43\phantom{.}\scalebox{0.72}{$\pm$0.68} & 34.70 \\
        (1.2) \textit{w/} Heuristic Weighting
        & 44.96\phantom{.}\scalebox{0.72}{$\pm$0.98} & 38.75\phantom{.}\scalebox{0.72}{$\pm$1.04} & 31.14\phantom{.}\scalebox{0.72}{$\pm$0.64} & 33.71\phantom{.}\scalebox{0.72}{$\pm$0.72} & 37.14 \\
        (1.3) \textit{w/} Mask After Wrong
        & 39.11\phantom{.}\scalebox{0.72}{$\pm$1.08} & 31.59\phantom{.}\scalebox{0.72}{$\pm$1.14} & 26.56\phantom{.}\scalebox{0.72}{$\pm$0.67} & 30.12\phantom{.}\scalebox{0.72}{$\pm$0.80} & 31.85 \\
        (1.4) \textit{w/o} Weight Clipping
        & 45.78\phantom{.}\scalebox{0.72}{$\pm$1.16} & 37.93\phantom{.}\scalebox{0.72}{$\pm$1.21} & 33.57\phantom{.}\scalebox{0.72}{$\pm$0.74} & 35.12\phantom{.}\scalebox{0.72}{$\pm$0.79} & 38.10 \\

        \midrule
        \rowcolor[HTML]{E0F3F8}
        \multicolumn{6}{c}{\textbf{Training Objective}}\\
        (2.1) \textit{w/o} GRPO
        & 48.87\phantom{.}\scalebox{0.72}{$\pm$0.93} & 39.73\phantom{.}\scalebox{0.72}{$\pm$1.02} & 35.89\phantom{.}\scalebox{0.72}{$\pm$0.62} & 38.62\phantom{.}\scalebox{0.72}{$\pm$0.71} & 40.78 \\
        (2.2) \textit{w/o} Step-wise OPD
        & 25.63\phantom{.}\scalebox{0.72}{$\pm$1.03} & 21.67\phantom{.}\scalebox{0.72}{$\pm$1.12} & 33.55\phantom{.}\scalebox{0.72}{$\pm$0.76} & 20.70\phantom{.}\scalebox{0.72}{$\pm$0.87} & 25.39 \\

        \midrule

        \textbf{\ourmodel}
        & \textbf{50.83}\phantom{.}\scalebox{0.72}{$\pm$1.15} & \textbf{41.72}\phantom{.}\scalebox{0.72}{$\pm$1.24} & \textbf{38.72}\phantom{.}\scalebox{0.72}{$\pm$0.73}
        & \textbf{40.63}\phantom{.}\scalebox{0.72}{$\pm$0.91} & \textbf{42.98} \\

        \bottomrule
    \end{tabular}%
    \end{adjustbox}

    \endgroup

    \vspace{2pt}
    \vspace{-10pt}
\end{table*}

\subsection{Ablation Study}

To evaluate each component in \ourmodel, we develop six variants from two perspectives on the Qwen3-1.7B student. For \textit{Step-wise Reweighting}, we consider: (1.1)~\textit{w/ Uniform Weighting}, which fixes $w_k=1$ for all steps; (1.2)~\textit{w/ Heuristic Weighting}, which applies a fixed exponential decay $w_k = \gamma^{k - k_{\mathrm{err}}}$ starting from the first erroneous tool call at step $k_{\mathrm{err}}$, with $\gamma$ fixed as 0.9; (1.3)~\textit{w/ Mask After Wrong}, which zeros out the OPD signal for all steps after the first tool-call error; and (1.4)~\textit{w/o Weight Clipping}, which removes the upper-bound clipping $\delta$ in~\cref{eq:adaptive_weight}. For \textit{Training Objective}, we ablate the two terms in \cref{eq:training_object}: (2.1)~\textit{w/o GRPO}, which removes $\mathcal{L}_{\mathrm{GRPO}}$; and (2.2)~\textit{w/o Step-wise OPD}, which removes $\mathcal{L}_{\mathrm{OPD}}^{\mathrm{step}}$. The results are reported in Table~\ref{tab:ablation}, from which we observe:

\begin{itemize}[leftmargin=*]

\item \textbf{Obs 3: The adaptive step-wise reweighting is critical, and neither static nor heuristic alternatives can substitute it.}
Replacing our adaptive reweighting mechanism with uniform weighting (variant 1.1) causes a notable drop to 34.70\% average, confirming that treating all steps equally exposes training to misleading supervision from corrupted states. Heuristic decay (variant 1.2) partially mitigates this issue (37.14\%) by down-weighting later steps, yet still falls behind because a fixed schedule cannot capture non-monotonic divergence patterns where the student recovers alignment after earlier errors. Hard masking (variant 1.3) performs worst (31.85\%), as it discards all subsequent supervision after a single error, wasting informative signals from partially correct trajectories and preventing the student from recovering. Removing weight clipping (variant 1.4) yields 38.10\%, indicating that unbounded amplification introduces training instability. 

\item \textbf{Obs 4: Both the GRPO and step-wise OPD components are indispensable, serving complementary roles.}
Removing GRPO (variant 2.1) reduces the average to 40.78\%, showing that outcome-level rewards remain valuable for steering exploration with dense distillation. Removing step-wise OPD (variant 2.2) causes a far more severe drop to 25.39\%, 
which confirms that sparse rewards alone are insufficient for stable TIR learning in small models. The pronounced asymmetry reveals their complementary nature: step-wise OPD provides the primary fine-grained guidance, while GRPO broadens trajectory-level exploration beyond the teacher distribution.

\end{itemize}

\subsection{Scalability and Generalization}

To evaluate the generalization of \ourmodel~to different teacher models, we conduct experiments using an additional Qwen3 teacher (14B) to distill students at both the 0.6B and 1.7B scales. The results are visualized in Figure~\ref{rq3}, from which we draw the following observation:

\begin{itemize}[leftmargin=*]
\item \textbf{Obs 5: \ourmodel~consistently benefits from stronger teachers, whereas OPD suffers from the increased capacity gap.}
An interesting finding emerges when comparing teacher choices under OPD: on the 0.6B student, switching from the 4B to the 14B teacher actually \emph{degrades} average accuracy by a notable margin. This reveals that a larger student-teacher capacity gap amplifies distribution mismatch, causing uniform distillation to propagate increasingly unreliable supervision. In contrast, \ourmodel~with the 14B teacher consistently outperforms its 4B counterpart on both student scales, confirming that our step-wise reweighting effectively harnesses the richer knowledge from teachers while suppressing the harmful signals introduced by the wider gap.

\begin{figure*}[t]
\centering
\includegraphics[width=1.0\linewidth]{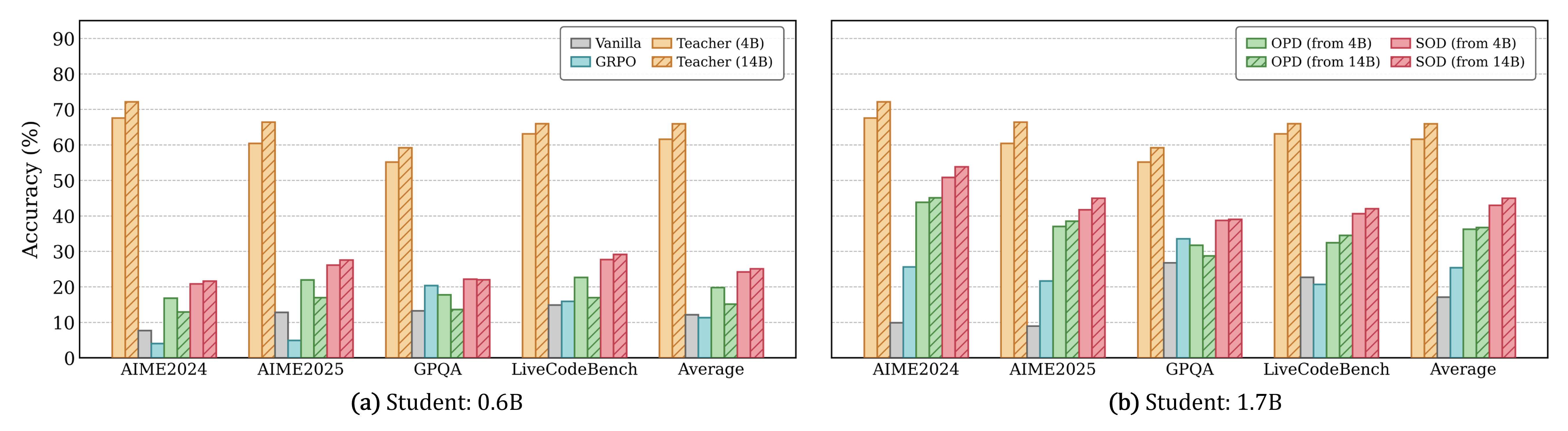}
\caption{Scalability of \ourmodel~across different student-teacher configurations.}
\label{rq3}
\end{figure*}

\end{itemize}

\begin{figure*}[h]
\centering
\includegraphics[width=1.0\linewidth]{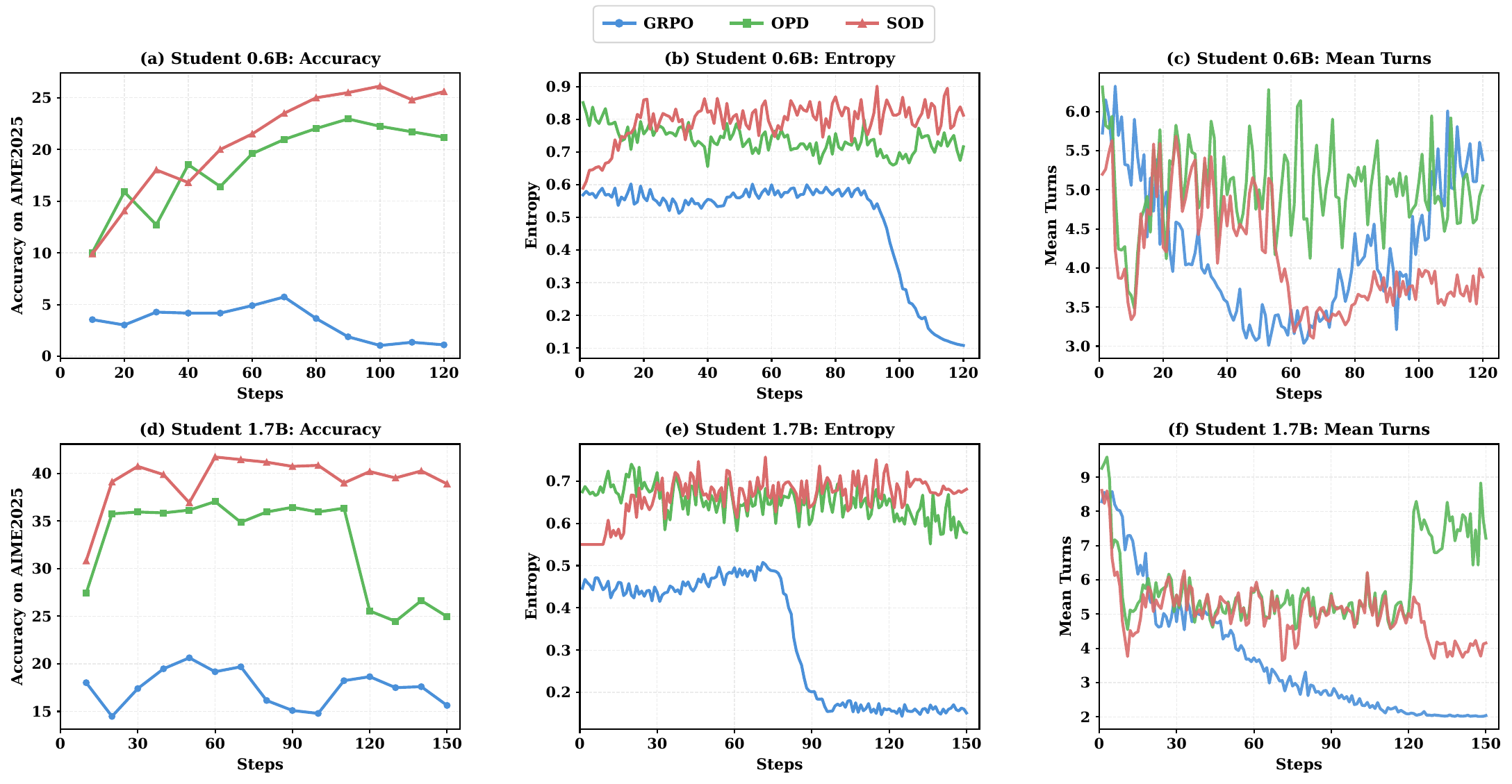}
\caption{\textbf{Training dynamics across methods on 0.6B and 1.7B student models}. We track accuracy on AIME2025 (left), policy entropy (middle), and mean tool-calling turns (right) throughout training.}
\label{rq4}
\end{figure*}

\subsection{Dynamic training analysis}

To understand how \ourmodel shapes the learning dynamics of agentic reasoning, we monitor three key metrics throughout training: task accuracy, policy entropy, and mean tool-calling turns. As shown in Figure~\ref{rq4}, we highlight the following observations:

\begin{itemize}[leftmargin=*]

\item \textbf{Obs 6: \ourmodel consistently outperforms GRPO by a large margin.} As shown in Figure~\ref{rq4}(a)(d), \ourmodel achieves significantly higher accuracy than GRPO across both model scales. This gap is driven by GRPO's severe entropy collapse (Figure~\ref{rq4}(b)(e)), where the policy loses exploration ability and degenerates into repetitive outputs, ultimately abandoning multi-step tool interactions entirely (Figure~\ref{rq4}(c)(f)), a fatal failure mode for TIR tasks.

\item \textbf{Obs 7: \ourmodel achieves more stable training than OPD.} 
While OPD maintains comparable entropy levels, its accuracy exhibits notable instability, particularly on 1.7B where performance peaks early then degrades significantly (Figure~\ref{rq4}(d)), suggesting that uniform distillation can destabilize later training. 
In contrast, \ourmodel sustains steady improvement without degradation. Additionally, \ourmodel uses fewer tool-calling turns than OPD (Figure~\ref{rq4}(c)(f)), indicating more efficient reasoning with fewer erroneous intermediate steps.

\end{itemize}

\subsection{Three Distillation Patterns of \ourmodel}
\label{rq5_body}
During training, \ourmodel exhibits three characteristic distillation patterns, as illustrated in Figure~\ref{rq5}:

\textbf{(1) Stable Pattern.} The student remains closely aligned with the teacher, \ie the step-level divergence $d_k$ is low. Therefore, adaptive weights stay high, allowing full utilization of teacher supervision along the trajectory. In this regime, the optimization closely resembles sufficient distillation.
\textbf{(2) Erroneous Pattern.} Persistent deviations increase step-level divergence, often caused by incorrect intermediate reasoning or tool usage. Adaptive weights are progressively reduced, which suppresses potentially misleading supervision and stabilizes training under such corrupted states.
\textbf{(3) Recovery Pattern.} The student initially deviates but then recovers. Weights are attenuated during high-divergence steps and restored later, enabling effective guidance on corrected steps and preserving useful supervision after recovery. This behavior allows the model to avoid over-penalizing temporary errors while still benefiting from supervision once alignment is re-established.

We also provide further analysis, \ie the quantitative analysis of distillation pattern distribution according training steps in~\cref{app:futher_casestudy}.
Detailed examples are provided in~\Cref{app:casestudy}.

\begin{figure*}[t]
\centering
\includegraphics[width=1.0\linewidth]{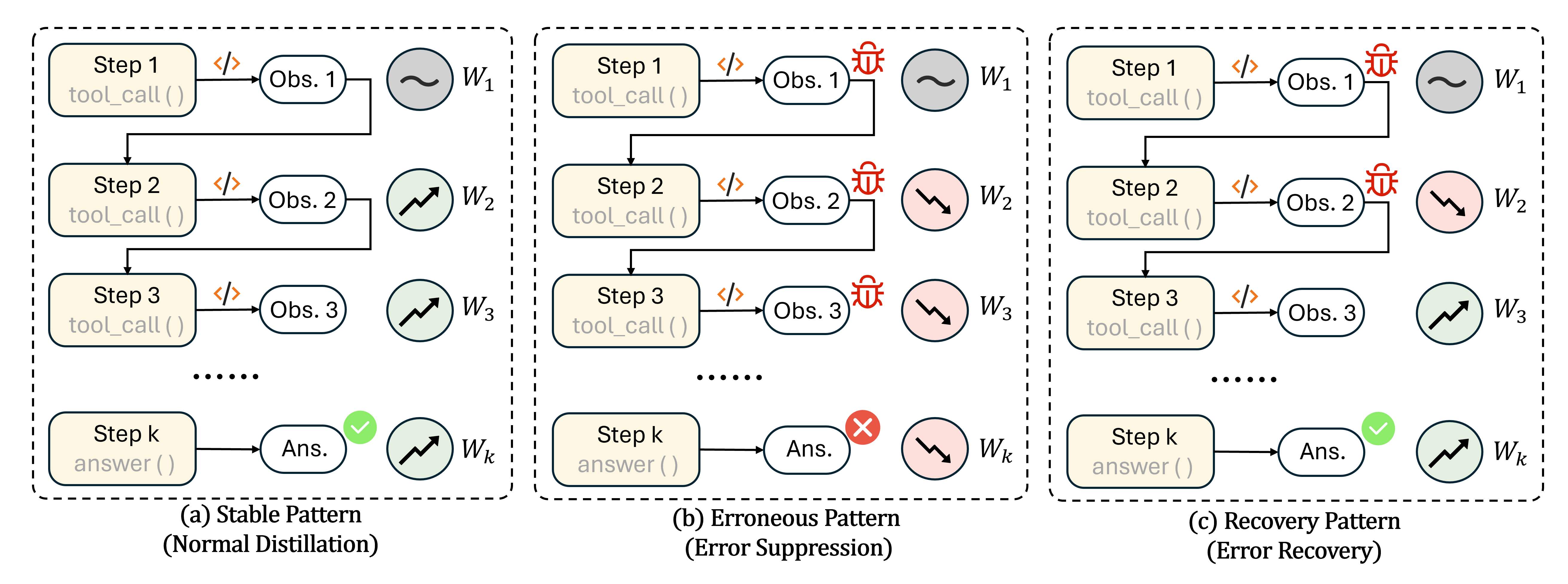}
\caption{\textbf{Three distillation patterns of \ourmodel}.
}
\vspace{-1 em}
\label{rq5}
\end{figure*}

\section{Limitations}
\label{sec:limitations}

Our work has two main limitations. (1) We focus on agentic TIR tasks with a python code interpreter, which provides a clean and reproducible execution environment that isolates tool-induced state drift without confounding factors from complex API configurations. Other agentic settings (\eg web browsing, API calls) may exhibit different drift patterns and are worth exploring. (2) All experiments use the Qwen3~\citep{qwen2025qwen3} model family, chosen for its strong and stable performance across scales, open availability at multiple sizes, and wide adoption in recent OPD studies~\citep{opsd,tip,eopd,li2026rethinking} that facilitates fair comparison. Validating on other model families remains valuable. Due to resource constraints, we leave these explorations to future work.

\section{Conclusion}\label{sec:conclusion}

In this work, we study on-policy distillation for small language model agents under tool-integrated reasoning. We identify that tool-induced state transitions can rapidly push the student into out-of-distribution regions where teacher supervision becomes unreliable. To address this, we propose \ourmodel, a step-wise on-policy distillation framework that adaptively reweights distillation strength based on student–teacher divergence, preserving informative signals in aligned regions while attenuating misleading guidance under large drift. Extensive experiments across math, science, and code benchmarks demonstrate consistent improvements in both training stability and final performance. Our findings suggest that effective agent distillation requires supervision mechanisms that adapt to evolving state distributions, offering a general principle for robust training in agentic systems.

\newpage

\bibliographystyle{unsrtnat}   %
\bibliography{reference}

@article{zou2025reasonflux-prm,
  title={ReasonFlux-PRM: Trajectory-Aware PRMs for Long Chain-of-Thought Reasoning in LLMs},
  author={Zou, Jiaru and Yang, Ling and Gu, Jingwen and Qiu, Jiahao and Shen, Ke and He, Jingrui and Wang, Mengdi},
  journal={The Thirty-ninth Annual Conference on Neural Information Processing Systems},
  year={2025}
}

@article{yu2025dapo,
  title={Dapo: An open-source llm reinforcement learning system at scale},
  author={Yu, Qiying and Zhang, Zheng and Zhu, Ruofei and Yuan, Yufeng and Zuo, Xiaochen and Yue, Yu and Dai, Weinan and Fan, Tiantian and Liu, Gaohong and Liu, Lingjun and others},
  journal={arXiv preprint arXiv:2503.14476},
  year={2025}
}

@article{he2025skywork-or1,
  title={Skywork open reasoner 1 technical report},
  author={He, Jujie and Liu, Jiacai and Liu, Chris Yuhao and Yan, Rui and Wang, Chaojie and Cheng, Peng and Zhang, Xiaoyu and Zhang, Fuxiang and Xu, Jiacheng and Shen, Wei and Li, Siyuan and Zeng, Liang and Wei, Tianwen and Cheng, Cheng and An, Bo and Liu, Yang and Zhou, Yahui},
  journal={arXiv preprint arXiv:2505.22312},
  year={2025}
}

@article{fan2025megascience,
  title={MegaScience: Pushing the Frontiers of Post-Training Datasets for Science Reasoning},
  author={Fan, Run-Ze and Wang, Zengzhi and Liu, Pengfei},
  journal={arXiv preprint arXiv:2507.16812},
  year={2025}
}

@article{demy,
  title={Demystifying reinforcement learning in agentic reasoning},
  author={Yu, Zhaochen and Yang, Ling and Zou, Jiaru and Yan, Shuicheng and Wang, Mengdi},
  journal={arXiv preprint arXiv:2510.11701},
  year={2025}
}

@inproceedings{rein2024gpqa,
  title={Gpqa: A graduate-level google-proof q\&a benchmark},
  author={Rein, David and Hou, Betty Li and Stickland, Asa Cooper and Petty, Jackson and Pang, Richard Yuanzhe and Dirani, Julien and Michael, Julian and Bowman, Samuel R},
  booktitle={First Conference on Language Modeling},
  year={2024}
}

@inproceedings{jainlivecodebench,
  title={LiveCodeBench: Holistic and Contamination Free Evaluation of Large Language Models for Code},
  author={Jain, Naman and Han, King and Gu, Alex and Li, Wen-Ding and Yan, Fanjia and Zhang, Tianjun and Wang, Sida and Solar-Lezama, Armando and Sen, Koushik and Stoica, Ion},
  booktitle={The Thirteenth International Conference on Learning Representations}
}

@article{feng2025retool,
  title={Retool: Reinforcement learning for strategic tool use in llms},
  author={Feng, Jiazhan and Huang, Shijue and Qu, Xingwei and Zhang, Ge and Qin, Yujia and Zhong, Baoquan and Jiang, Chengquan and Chi, Jinxin and Zhong, Wanjun},
  journal={arXiv preprint arXiv:2504.11536},
  year={2025}
}

@inproceedings{gu2024minillm,
  title={Minillm: Knowledge distillation of large language models},
  author={Gu, Yuxian and Dong, Li and Wei, Furu and Huang, Minlie},
  booktitle={The twelfth international conference on learning representations},
  year={2024}
}

@article{yang2026learning,
  title={Learning beyond teacher: Generalized on-policy distillation with reward extrapolation},
  author={Yang, Wenkai and Liu, Weijie and Xie, Ruobing and Yang, Kai and Yang, Saiyong and Lin, Yankai},
  journal={arXiv preprint arXiv:2602.12125},
  year={2026}
}

@article{li2026rethinking,
  title={Rethinking On-Policy Distillation of Large Language Models: Phenomenology, Mechanism, and Recipe},
  author={Li, Yaxuan and Zuo, Yuxin and He, Bingxiang and Zhang, Jinqian and Xiao, Chaojun and Qian, Cheng and Yu, Tianyu and Gao, Huan-ang and Yang, Wenkai and Liu, Zhiyuan and others},
  journal={arXiv preprint arXiv:2604.13016},
  year={2026}
}

@article{christiano2017deep,
  title={Deep reinforcement learning from human preferences},
  author={Christiano, Paul F and Leike, Jan and Brown, Tom and Martic, Miljan and Legg, Shane and Amodei, Dario},
  journal={Advances in neural information processing systems},
  volume={30},
  year={2017}
}

@article{ziegler2019fine,
  title={Fine-tuning language models from human preferences},
  author={Ziegler, Daniel M and Stiennon, Nisan and Wu, Jeffrey and Brown, Tom B and Radford, Alec and Amodei, Dario and Christiano, Paul and Irving, Geoffrey},
  journal={arXiv preprint arXiv:1909.08593},
  year={2019}
}

@article{chen2023fireact,
  title={Fireact: Toward language agent fine-tuning},
  author={Chen, Baian and Shu, Chang and Shareghi, Ehsan and Collier, Nigel and Narasimhan, Karthik and Yao, Shunyu},
  journal={arXiv preprint arXiv:2310.05915},
  year={2023}
}

@article{yang2024swe,
  title={Swe-agent: Agent-computer interfaces enable automated software engineering},
  author={Yang, John and Jimenez, Carlos E and Wettig, Alexander and Lieret, Kilian and Yao, Shunyu and Narasimhan, Karthik and Press, Ofir},
  journal={Advances in Neural Information Processing Systems},
  volume={37},
  pages={50528--50652},
  year={2024}
}

@article{bai2024digirl,
  title={Digirl: Training in-the-wild device-control agents with autonomous reinforcement learning},
  author={Bai, Hao and Zhou, Yifei and Cemri, Mert and Pan, Jiayi and Suhr, Alane and Levine, Sergey and Kumar, Aviral},
  journal={Advances in Neural Information Processing Systems},
  volume={37},
  pages={12461--12495},
  year={2024}
}

@article{qi2024webrl,
  title={Webrl: Training llm web agents via self-evolving online curriculum reinforcement learning},
  author={Qi, Zehan and Liu, Xiao and Iong, Iat Long and Lai, Hanyu and Sun, Xueqiao and Zhao, Wenyi and Yang, Yu and Yang, Xinyue and Sun, Jiadai and Yao, Shuntian and others},
  journal={arXiv preprint arXiv:2411.02337},
  year={2024}
}

@article{zhou2024archer,
  title={Archer: Training language model agents via hierarchical multi-turn rl, 2024},
  author={Zhou, Yifei and Zanette, Andrea and Pan, Jiayi and Levine, Sergey and Kumar, Aviral},
  journal={URL https://arxiv. org/abs/2402.19446},
  year={2024}
}

@article{chen2025reinforcement,
  title={Reinforcement learning for long-horizon interactive llm agents},
  author={Chen, Kevin and Cusumano-Towner, Marco and Huval, Brody and Petrenko, Aleksei and Hamburger, Jackson and Koltun, Vladlen and Kr{\"a}henb{\"u}hl, Philipp},
  journal={arXiv preprint arXiv:2502.01600},
  year={2025}
}

@article{wang2026rlanything,
  title={RLAnything: Forge Environment, Policy, and Reward Model in Completely Dynamic RL System},
  author={Wang, Yinjie and Xie, Tianbao and Shen, Ke and Wang, Mengdi and Yang, Ling},
  journal={arXiv preprint arXiv:2602.02488},
  year={2026}
}

@article{wang2025co,
  title={Co-evolving llm coder and unit tester via reinforcement learning},
  author={Wang, Yinjie and Yang, Ling and Tian, Ye and Shen, Ke and Wang, Mengdi},
  journal={arXiv preprint arXiv:2506.03136},
  year={2025}
}

@article{schulman2017proximal,
  title={Proximal policy optimization algorithms},
  author={Schulman, John and Wolski, Filip and Dhariwal, Prafulla and Radford, Alec and Klimov, Oleg},
  journal={arXiv preprint arXiv:1707.06347},
  year={2017}
}

@article{grpo,
  title={Deepseekmath: Pushing the limits of mathematical reasoning in open language models},
  author={Shao, Zhihong and Wang, Peiyi and Zhu, Qihao and Xu, Runxin and Song, Junxiao and Bi, Xiao and Zhang, Haowei and Zhang, Mingchuan and Li, YK and Wu, Yang and others},
  journal={arXiv preprint arXiv:2402.03300},
  year={2024}
}

@inproceedings{opd,
  title={On-policy distillation of language models: Learning from self-generated mistakes},
  author={Agarwal, Rishabh and Vieillard, Nino and Zhou, Yongchao and Stanczyk, Piotr and Garea, Sabela Ramos and Geist, Matthieu and Bachem, Olivier},
  booktitle={The twelfth international conference on learning representations},
  year={2024}
}

@article{opsd,
  title={Self-Distilled Reasoner: On-Policy Self-Distillation for Large Language Models},
  author={Zhao, Siyan and Xie, Zhihui and Liu, Mengchen and Huang, Jing and Pang, Guan and Chen, Feiyu and Grover, Aditya},
  journal={arXiv preprint arXiv:2601.18734},
  year={2026}
}

@article{xi2023rise,
  title={The Rise and Potential of Large Language Model Based Agents: A Survey},
  author={Xi, Zhiheng and Chen, Wenxiang and Guo, Xin and He, Wei and Ding, Yiwen and Hong, Boyang and Zhang, Ming and Wang, Junzhe and Jin, Senjie and Zhou, Enyu and others},
  journal={arXiv preprint arXiv:2309.07864},
  year={2023}
}

@inproceedings{react,
  title={ReAct: Synergizing Reasoning and Acting in Language Models},
  author={Yao, Shunyu and Zhao, Jeffrey and Yu, Dian and Du, Nan and Shafran, Izhak and Narasimhan, Karthik R. and Cao, Yuan},
  booktitle={ICLR},
  year={2023}
}

@inproceedings{schick2023toolformer,
  title={Toolformer: Language Models Can Teach Themselves to Use Tools},
  author={Schick, Timo and Dwivedi-Yu, Jane and Dess{\`\i}, Roberto and Raileanu, Roberta and Lomeli, Maria and Hambro, Eric and Zettlemoyer, Luke and Cancedda, Nicola and Scialom, Thomas},
  booktitle={NeurIPS},
  year={2023}
}

@article{singh2025agentic,
  title={Agentic Reasoning and Tool Integration for LLMs via Reinforcement Learning},
  author={Singh, Joykirat and Magazine, Raghav and Pandya, Yash and Nambi, Akshay},
  journal={arXiv preprint arXiv:2505.01441},
  year={2025}
}

@article{xu2024ondevice,
  title={On-Device Language Models: A Comprehensive Review},
  author={Xu, Jiajun and Li, Zhiyuan and Chen, Wei and Wang, Qun and Gao, Xin and Cai, Qi and Ling, Ziyuan},
  journal={arXiv preprint arXiv:2409.00088},
  year={2024}
}

@article{xu2024kd_llm_survey,
  title={A Survey on Knowledge Distillation of Large Language Models},
  author={Xu, Xiaohan and Li, Ming and Tao, Chongyang and Shen, Tao and Cheng, Reynold and Li, Jinyang and Xu, Can and Tao, Dacheng and Zhou, Tianyi},
  journal={arXiv preprint arXiv:2402.13116},
  year={2024}
}

@article{li2025torl,
  title={ToRL: Scaling Tool-Integrated RL},
  author={Li, Xuefeng and Zou, Haoyang and Liu, Pengfei},
  journal={arXiv preprint arXiv:2503.23383},
  year={2025}
}

@article{qian2025toolrl,
  title={ToolRL: Reward is All Tool Learning Needs},
  author={Qian, Cheng and Acikgoz, Emre Can and He, Qi and Wang, Hongru and Chen, Xiusi and Hakkani-T{\"u}r, Dilek and Tur, Gokhan and Ji, Heng},
  journal={arXiv preprint arXiv:2504.13958},
  year={2025}
}

@article{rainone2025replacing,
  title={Replacing thinking with tool usage enables reasoning in small language models},
  author={Rainone, Corrado and Bakker, Tim and Memisevic, Roland},
  journal={arXiv preprint arXiv:2507.05065},
  year={2025}
}

@article{deepseek2025r1,
  title={DeepSeek-R1: Incentivizing Reasoning Capability in LLMs via Reinforcement Learning},
  author={DeepSeek-AI},
  journal={arXiv preprint arXiv:2501.12948},
  year={2025}
}

@article{jin2025searchr1,
  title={Search-R1: Training LLMs to Reason and Leverage Search Engines with Reinforcement Learning},
  author={Jin, Bowen and Zeng, Hansi and Yue, Zhenrui and Wang, Dong and Zamani, Hamed and Han, Jiawei},
  journal={arXiv preprint arXiv:2503.09516},
  year={2025}
}

@article{song2025r1searcher,
  title={R1-Searcher: Incentivizing the Search Capability in LLMs via Reinforcement Learning},
  author={Song, Huatong and Jiang, Jinhao and Min, Yingqian and Chen, Jie and Chen, Zhipeng and Zhao, Wayne Xin and Fang, Lei and Wen, Ji-Rong},
  journal={arXiv preprint arXiv:2503.05592},
  year={2025}
}

@article{kepo,
  title={KEPO: Knowledge-Enhanced Preference Optimization for Reinforcement Learning with Reasoning},
  author={Yang, Fan and Meng, Rui and Di Qi, Trudi and Ezzati, Ali and Wen, Yuxin},
  journal={arXiv preprint arXiv:2602.00400},
  year={2026}
}

@article{qwen2025qwen3,
  title={Qwen3 Technical Report},
  author={Qwen Team},
  journal={arXiv preprint arXiv:2505.09388},
  year={2025}
}

@article{xue2025simpletir,
  title={Simpletir: End-to-end reinforcement learning for multi-turn tool-integrated reasoning},
  author={Xue, Zhenghai and Zheng, Longtao and Liu, Qian and Li, Yingru and Zheng, Xiaosen and Ma, Zejun and An, Bo},
  journal={arXiv preprint arXiv:2509.02479},
  year={2025}
}

@inproceedings{xu2026rlkd,
  title={RLKD: Distilling LLMs’ Reasoning via Reinforcement Learning},
  author={Xu, Shicheng and Pang, Liang and Zhu, Yunchang and Gu, Jia and Wei, Zihao and Deng, Jingcheng and Pan, Feiyang and Shen, Huawei and Cheng, Xueqi},
  booktitle={Proceedings of the AAAI Conference on Artificial Intelligence},
  volume={40},
  number={40},
  pages={34151--34159},
  year={2026}
}

@article{fu2026revisiting,
  title={Revisiting On-Policy Distillation: Empirical Failure Modes and Simple Fixes},
  author={Fu, Yuqian and Huang, Haohuan and Jiang, Kaiwen and Zhu, Yuanheng and Zhao, Dongbin},
  journal={arXiv preprint arXiv:2603.25562},
  year={2026}
}

@inproceedings{ross2011reduction,
  title={A Reduction of Imitation Learning and Structured Prediction to No-Regret Online Learning},
  author={Ross, St{\'e}phane and Gordon, Geoffrey J. and Bagnell, Drew},
  booktitle={AISTATS},
  year={2011}
}

@article{gudibande2024false,
  title={The False Promise of Imitating Proprietary LLMs},
  author={Gudibande, Arnav and Wallace, Eric and Snell, Charlie and Geng, Xinyang and Liu, Hao and Abbeel, Pieter and Levine, Sergey and Song, Dawn},
  journal={arXiv preprint arXiv:2305.15717},
  year={2023}
}

@article{opdsurvey,
  title={A Survey of On-Policy Distillation for Large Language Models},
  author={Song, Mingyang and Zheng, Mao},
  journal={arXiv preprint arXiv:2604.00626},
  year={2026}
}

@article{bousselham2025vold,
  title={VOLD: Reasoning Transfer from LLMs to Vision-Language Models via On-Policy Distillation},
  author={Bousselham, Walid and Kuehne, Hilde and Schmid, Cordelia},
  journal={arXiv preprint arXiv:2510.23497},
  year={2025}
}

@article{srpo,
  title={Unifying Group-Relative and Self-Distillation Policy Optimization via Sample Routing},
  author={Li, Gengsheng and Yang, Tianyu and Fang, Junfeng and Song, Mingyang and Zheng, Mao and Guo, Haiyun and Zhang, Dan and Wang, Jinqiao and Chua, Tat-Seng},
  journal={arXiv preprint arXiv:2604.02288},
  year={2026}
}

@article{sdrlvr,
  title={Self-Distilled RLVR},
  author={Yang, Chenxu and Qin, Chuanyu and Si, Qingyi and Chen, Minghui and Gu, Naibin and Yao, Dingyu and Lin, Zheng and Wang, Weiping and Wang, Jiaqi and Duan, Nan},
  journal={arXiv preprint arXiv:2604.03128},
  year={2026}
}

@article{skillsd,
  title={Skill-SD: Skill-Conditioned Self-Distillation for Multi-turn LLM Agents},
  author={Wang, Hao and Wang, Guozhi and Xiao, Han and Zhou, Yufeng and Pan, Yue and Wang, Jichao and Xu, Ke and Wen, Yafei and Ruan, Xiaohu and Chen, Xiaoxin and others},
  journal={arXiv preprint arXiv:2604.10674},
  year={2026}
}

@article{wang2026openclaw,
  title={Openclaw-rl: Train any agent simply by talking},
  author={Wang, Yinjie and Chen, Xuyang and Jin, Xiaolong and Wang, Mengdi and Yang, Ling},
  journal={arXiv preprint arXiv:2603.10165},
  year={2026}
}

@inproceedings{verl,
  title={Hybridflow: A flexible and efficient rlhf framework},
  author={Sheng, Guangming and Zhang, Chi and Ye, Zilingfeng and Wu, Xibin and Zhang, Wang and Zhang, Ru and Peng, Yanghua and Lin, Haibin and Wu, Chuan},
  booktitle={Proceedings of the Twentieth European Conference on Computer Systems},
  pages={1279--1297},
  year={2025}
}

@article{opcd,
  title={On-policy context distillation for language models},
  author={Ye, Tianzhu and Dong, Li and Wu, Xun and Huang, Shaohan and Wei, Furu},
  journal={arXiv preprint arXiv:2602.12275},
  year={2026}
}

@article{gad,
  title={Black-Box On-Policy Distillation of Large Language Models},
  author={Ye, Tianzhu and Dong, Li and Chi, Zewen and Wu, Xun and Huang, Shaohan and Wei, Furu},
  journal={arXiv preprint arXiv:2511.10643},
  year={2025}
}

@article{veto,
  title={Stable On-Policy Distillation through Adaptive Target Reformulation},
  author={Jang, Ijun and Yeom, Jewon and Yeo, Juan and Lim, Hyunggu and Kim, Taesup},
  journal={arXiv preprint arXiv:2601.07155},
  year={2026}
}

@article{opsd2,
  title={Privileged Information Distillation for Language Models},
  author={Penaloza, Emiliano and Vattikonda, Dheeraj and Gontier, Nicolas and Lacoste, Alexandre and Charlin, Laurent and Caccia, Massimo},
  journal={arXiv preprint arXiv:2602.04942},
  year={2026}
}

@article{sdpo,
  title={Reinforcement Learning via Self-Distillation},
  author={H{\"u}botter, Jonas and L{\"u}beck, Frederike and Behric, Lejs and Baumann, Anton and Bagatella, Marco and Marta, Daniel and Hakimi, Ido and Shenfeld, Idan and Buening, Thomas Kleine and Guestrin, Carlos and others},
  journal={arXiv preprint arXiv:2601.20802},
  year={2026}
}

@article{sdft,
  title={Self-Distillation Enables Continual Learning},
  author={Shenfeld, Idan and Damani, Mehul and H{\"u}botter, Jonas and Agrawal, Pulkit},
  journal={arXiv preprint arXiv:2601.19897},
  year={2026}
}

@article{coa,
  title={Chain-of-agents: End-to-end agent foundation models via multi-agent distillation and agentic rl},
  author={Li, Weizhen and Lin, Jianbo and Jiang, Zhuosong and Cao, Jingyi and Liu, Xinpeng and Zhang, Jiayu and Huang, Zhenqiang and Chen, Qianben and Sun, Weichen and Wang, Qiexiang and others},
  journal={arXiv preprint arXiv:2508.13167},
  year={2025}
}

@article{AgentDistillation,
  title={Distilling llm agent into small models with retrieval and code tools},
  author={Kang, Minki and Jeong, Jongwon and Lee, Seanie and Cho, Jaewoong and Hwang, Sung Ju},
  journal={arXiv preprint arXiv:2505.17612},
  year={2025}
}

@article{Agentdistill,
  title={Agentdistill: Training-free agent distillation with generalizable mcp boxes},
  author={Qiu, Jiahao and Juan, Xinzhe and Wang, Yimin and Yang, Ling and Qi, Xuan and Zhang, Tongcheng and Guo, Jiacheng and Lu, Yifu and Yao, Zixin and Wang, Hongru and others},
  journal={arXiv preprint arXiv:2506.14728},
  year={2025}
}

@article{sad,
  title={Structured agent distillation for large language model},
  author={Liu, Jun and Kong, Zhenglun and Dong, Peiyan and Yang, Changdi and Li, Tianqi and Tang, Hao and Yuan, Geng and Niu, Wei and Zhang, Wenbin and Zhao, Pu and others},
  journal={arXiv preprint arXiv:2505.13820},
  year={2025}
}

@article{oresearcher,
  title={O-Researcher: An Open Ended Deep Research Model via Multi-Agent Distillation and Agentic RL},
  author={Yao, Yi and Zhu, He and Wang, Piaohong and Ren, Jincheng and Yang, Xinlong and Chen, Qianben and Li, Xiaowan and Shi, Dingfeng and Li, Jiaxian and Wang, Qiexiang and others},
  journal={arXiv preprint arXiv:2601.03743},
  year={2026}
}

@article{eopd,
  title={Entropy-Aware On-Policy Distillation of Language Models},
  author={Jin, Woogyeol and Min, Taywon and Yang, Yongjin and Kadhe, Swanand Ravindra and Zhou, Yi and Wei, Dennis and Baracaldo, Nathalie and Lee, Kimin},
  journal={arXiv preprint arXiv:2603.07079},
  year={2026}
}

@article{ropd,
  title={Scaling reasoning efficiently via relaxed on-policy distillation},
  author={Ko, Jongwoo and Abdali, Sara and Kim, Young Jin and Chen, Tianyi and Cameron, Pashmina},
  journal={arXiv preprint arXiv:2603.11137},
  year={2026}
}

@article{lightOPD,
  title={Lightning OPD: Efficient Post-Training for Large Reasoning Models with Offline On-Policy Distillation},
  author={Wu, Yecheng and Han, Song and Cai, Hai},
  journal={arXiv preprint arXiv:2604.13010},
  year={2026}
}

@article{hybridOPD,
  title={Hybrid Policy Distillation for LLMs},
  author={Zhu, Wenhong and Xie, Ruobing and Wang, Rui and Liu, Pengfei},
  journal={arXiv preprint arXiv:2604.20244},
  year={2026}
}

@article{scope,
  title={SCOPE: Signal-Calibrated On-Policy Distillation Enhancement with Dual-Path Adaptive Weighting},
  author={Zheng, Binbin and Ma, Xing and Liang, Yiheng and Ruan, Jingqing and Fu, Xiaoliang and Lin, Kepeng and Zhu, Benchang and Zeng, Ke and Cai, Xunliang},
  journal={arXiv preprint arXiv:2604.10688},
  year={2026}
}

@article{tip,
  title={TIP: Token Importance in On-Policy Distillation},
  author={Xu, Yuanda and Sang, Hejian and Zhou, Zhengze and He, Ran and Wang, Zhipeng and Geramifard, Alborz},
  journal={arXiv preprint arXiv:2604.14084},
  year={2026}
}

@article{dqchen,
  title={Self-Distillation Zero: Self-Revision Turns Binary Rewards into Dense Supervision},
  author={He, Yinghui and Kaur, Simran and Bhaskar, Adithya and Yang, Yongjin and Liu, Jiarui and Ri, Narutatsu and Fowl, Liam and Panigrahi, Abhishek and Chen, Danqi and Arora, Sanjeev},
  journal={arXiv preprint arXiv:2604.12002},
  year={2026}
}

@inproceedings{muennighoff2025s1,
  title={s1: Simple test-time scaling},
  author={Muennighoff, Niklas and Yang, Zitong and Shi, Weijia and Li, Xiang Lisa and Fei-Fei, Li and Hajishirzi, Hannaneh and Zettlemoyer, Luke and Liang, Percy and Cand{\`e}s, Emmanuel and Hashimoto, Tatsunori B},
  booktitle={Proceedings of the 2025 Conference on Empirical Methods in Natural Language Processing},
  pages={20286--20332},
  year={2025}
}

@article{xia2025leetcodedataset,
  title={Leetcodedataset: A temporal dataset for robust evaluation and efficient training of code llms},
  author={Xia, Yunhui and Shen, Wei and Wang, Yan and Liu, Jason Klein and Sun, Huifeng and Wu, Siyue and Hu, Jian and Xu, Xiaolong},
  journal={arXiv preprint arXiv:2504.14655},
  year={2025}
}

@article{wang2026tcod,
  title={TCOD: Exploring Temporal Curriculum in On-Policy Distillation for Multi-turn Autonomous Agents},
  author={Wang, Jiaqi and Zhang, Wenhao and Shi, Weijie and Li, Yaliang and Cheng, James},
  journal={arXiv preprint arXiv:2604.24005},
  year={2026}
}

@article{chen2026soda,
  title={SODA: Semi On-Policy Black-Box Distillation for Large Language Models},
  author={Chen, Xiwen and Wang, Jingjing and Zhu, Wenhui and Qiu, Peijie and Dong, Xuanzhao and Sang, Hejian and Wang, Zhipeng and Geramifard, Alborz and Luo, Feng},
  journal={arXiv preprint arXiv:2604.03873},
  year={2026}
}

@inproceedings{chenglin2024mixed,
  title={Mixed distillation helps smaller language models reason better},
  author={Chenglin, Li and Chen, Qianglong and Li, Liangyue and Wang, Caiyu and Tao, Feng and Li, Yicheng and Chen, Zulong and Zhang, Yin},
  booktitle={Findings of the Association for Computational Linguistics: EMNLP 2024},
  pages={1673--1690},
  year={2024}
}

@article{fan2026agentprocessbench,
  title={AgentProcessBench: Diagnosing Step-Level Process Quality in Tool-Using Agents},
  author={Fan, Shengda and Ye, Xuyan and Huo, Yupeng and Chen, Zhi-Yuan and Guo, Yiju and Yang, Shenzhi and Yang, Wenkai and Ye, Shuqi and Chen, Jingwen and Chen, Haotian and others},
  journal={arXiv preprint arXiv:2603.14465},
  year={2026}
}

@article{lu2025onpolicydistillation,
  author = {Kevin Lu and Thinking Machines Lab},
  title = {On-Policy Distillation},
  journal = {Thinking Machines Lab: Connectionism},
  year = {2025},
  note = {https://thinkingmachines.ai/blog/on-policy-distillation},
  doi = {10.64434/tml.20251026},
}

\newpage
\appendix

\crefalias{section}{appendix}
\crefalias{subsection}{appendix}
\crefalias{subsubsection}{appendix}

\section*{Appendix}
\addcontentsline{toc}{section}{Appendix}
\section{Algorithmic Details of \ourmodel}
\label{app:algorithm}

We present the complete training procedure of \ourmodel{} in Algorithm~\ref{alg:step_opd}. At each iteration, the student generates on-policy trajectories with tool interactions, computes step-level divergence against the teacher, derives adaptive distillation weights, and updates the policy with a combined GRPO and step-wise OPD objective.

\begin{algorithm}[h]
\caption{\ourmodel: Step-wise On-policy Distillation for Small Language Model Agents}
\label{alg:step_opd}
\KwIn{Student policy $\pi_\theta$, teacher policy $\pi_{\mathrm{teacher}}$, tool environment $\mathcal{E}$, prompt set $\mathcal{X}$, group size $G$, smoothing constant $\epsilon$, weight cap $\delta$, clipping range $\epsilon_{\mathrm{clip}}$.}
\KwOut{Trained student policy $\pi_\theta$.}

\vspace{0.5 em}
\BlankLine
\textbf{Stage I: On-Policy Rollout with Tool Interaction} \\
\For{each prompt $x \in \mathcal{X}$}{
    \begin{itemize}[nosep,leftmargin=1.5em]
        \item For $i = 1, \ldots, G$: sample trajectory $\tau_i = (y_1, o_1, \ldots, y_K, o_K, y_{K+1})$ from $\pi_{\theta_{\mathrm{old}}}$ by interacting with $\mathcal{E}$.
        \item Compute outcome reward $r_i = r(\tau_i)$ for each trajectory.
        \item Compute group-relative advantage $\hat{A}_i$ via~\cref{eq:advantage}.
    \end{itemize}
}
\vspace{0.5 em}
\BlankLine
\textbf{Stage II: Student-teacher Divergence and Step-wise Reweighting} \\
\For{each trajectory $\tau_i$}{
    \begin{itemize}[nosep,leftmargin=1.5em]
        \item Partition model-generated tokens into reasoning steps $\{\mathcal{I}_k\}_{k=1}^{K+1}$ (excluding tool observation tokens).
        \item For each step $k = 1, \ldots, K{+}1$, compute step-level divergence $d_k$ via~\cref{eq:step_div}.
        \item Set $w_1 \leftarrow 1$.
        \item For $k = 2, \ldots, K{+}1$, compute adaptive weight $w_k$ via~\cref{eq:adaptive_weight}.
    \end{itemize}
}
\BlankLine
\vspace{0.5 em}
\textbf{Stage III: Combined Optimization} 
\begin{itemize}[nosep,leftmargin=1.5em]
    \item Compute GRPO loss $\mathcal{L}_{\mathrm{GRPO}}$ with clipped surrogate objective (~\cref{eq:grpo}) using advantages $\{\hat{A}_i\}$.
    \item Compute step-wise OPD loss $\mathcal{L}_{\mathrm{OPD}}^{\mathrm{step}}$ via~\cref{eq:step_opd} using weights $\{w_k\}$.
    \item Update policy parameters: \\[2pt]
        $\theta \leftarrow \theta - \eta \nabla_\theta \big( \mathcal{L}_{\mathrm{GRPO}} + \mathcal{L}_{\mathrm{OPD}}^{\mathrm{step}} \big)$.
    \item Synchronize old policy: $\theta_{\mathrm{old}} \leftarrow \theta$.
\end{itemize}
\end{algorithm}

\section{Experimental Setup}
\label{app:experimental_setup}

\subsection{Training Datasets.}
\label{app:datasets}
We provide information in detail of our training datasets including SFT dataset and RL dataset, which is following~\citet{demy}. For the prompts template for different benchmarks, please refer to~\cref{dataset_prompt}. We also provide statistical information of our training data in~\cref{tab:training_datasets}.

\begin{itemize}

\item \textbf{SFT dataset.} The SFT dataset is constructed from a mixture of curated and synthetic multi-turn problem-solving trajectories. Specifically, a teacher model (Qwen3-Coder-30B-A3B) is employed within an agent-based framework where the SandBoxFusion is used as the code interpreter to generate end-to-end interaction traces, where problems are drawn from three complementary sources: the s1-1k set~\citep{muennighoff2025s1}, a curated collection of 3k LeetCode problems, and a 2k multi-turn ReTool~\citep{feng2025retool} dataset. To ensure data quality, the generated trajectories for the latter two subsets are scored using ReasonFlux-PRM~\citep{zou2025reasonflux-prm}, and only the top-ranked subsets (1k each) are retained, together with the full s1-1k set, resulting in a final 3k high-quality SFT corpus.
\item \textbf{RL dataset.} The RL dataset is designed to emphasize diversity across domains in order to study its impact on training dynamics. It comprises 30k samples aggregated from multiple sources, including 17k mathematical reasoning problems from DAPO-Math~\citep{yu2025dapo}, a mixture of 4,902 math and 3,586 code tasks from Skywork-or1~\citep{he2025skywork-or1}, and an additional 3k science problems from MegaScience~\citep{fan2025megascience}.

\end{itemize}

\begin{table*}[!t]
\centering
\caption{
Summary of the training datasets used in our experiments.
}
\small
\setlength{\tabcolsep}{6pt}
\renewcommand{\arraystretch}{1.15}
\begin{tabular}{l l l c}
\toprule
\textbf{Dataset Type} & \textbf{Domain} & \textbf{Source} & \textbf{\# Samples} \\
\midrule

\multirow{3}{*}{\textbf{SFT Dataset}}
& Mathematical %
& s1-1k~\citep{muennighoff2025s1} 
& 1k \\

& Coding 
& LeetCode~\citep{xia2025leetcodedataset}
& 1k \\

& Tool-use %
& ReTool~\citep{feng2025retool} 
& 1k \\

\cmidrule(lr){2-4}
& \multicolumn{2}{l}{\textbf{Total}} 
& \textbf{3k} \\

\midrule

\multirow{3}{*}{\textbf{RL Dataset}}
& \multirow{2}{*}{Mathematical %
}
& DAPO-Math~\citep{yu2025dapo} 
& 17k \\

& 
& Skywork-or1 (Math)~\citep{he2025skywork-or1} 
& 4,902 \\

& Coding 
& Skywork-or1 (Code)~\citep{he2025skywork-or1} 
& 3,586 \\

& Scientific %
& MegaScience~\citep{fan2025megascience} 
& 3k \\

\cmidrule(lr){2-4}
& \multicolumn{2}{l}{\textbf{Total}} 
& $\sim \textbf{30k}$ \\

\bottomrule
\end{tabular}
\label{tab:training_datasets}
\end{table*}

\subsection{Benchmarks}
\label{app:benchmarks}
We provide detailed descriptions of the four benchmarks used in our evaluation.

\begin{itemize}

\item \textbf{AIME 2024/2025.} The American Invitational Mathematics Examination (AIME) is a prestigious mathematics competition administered by the Mathematical Association of America (MAA). Each year's examination consists of two sessions (AIME I and AIME II), each containing 15 problems, for a total of 30 problems per year. All answers are integers in the range $[0, 999]$, which enables unambiguous automatic evaluation via exact match. The problems span advanced topics including algebra, number theory, combinatorics, and geometry, requiring multi-step reasoning chains and creative problem-solving strategies that go well beyond pattern matching. We use the 2024 and 2025 editions as benchmarks: AIME 2024 was administered in February 2024, and AIME 2025 in February 2025. The use of recent competition problems mitigates the risk of data contamination, as these problems are unlikely to appear in pre-training corpora of models with earlier knowledge cutoffs. 

\item \textbf{GPQA-Diamond}~\citep{rein2024gpqa}. GPQA (Graduate-Level Google-Proof Q\&A) is a challenging multiple-choice question-answering benchmark consisting of questions authored by domain experts holding PhDs in biology, physics, and chemistry. The questions are deliberately designed to be ``Google-proof'', they cannot be answered through simple web searches, requiring instead genuine domain expertise and multi-step scientific reasoning. The dataset comprises three nested subsets: GPQA Extended (546 questions), GPQA Main (448 questions), and GPQA Diamond (198 questions). We adopt the Diamond subset, which is the highest-quality partition: it contains only questions where both independent domain expert validators answered correctly, while skilled non-expert validators (holding PhDs in other scientific fields, with unrestricted internet access) failed. Human performance baselines on the Diamond subset are approximately 65\% for domain experts and 34\% for non-experts, compared to a random baseline of 25\%. 

\item \textbf{LiveCodeBench}~\citep{jainlivecodebench}. LiveCodeBench is a continuously updated benchmark for evaluating the coding capabilities of large language models. Unlike static benchmarks such as HumanEval, LiveCodeBench mitigates data contamination by continuously collecting new problems from competitive programming platforms (LeetCode, AtCoder, and Codeforces) and annotating each problem with its release date. This temporal annotation enables contamination-aware evaluation by restricting the test set to problems released after a model's training data cutoff. We use the v6 release, which contains 1,055 problems spanning May 2023 to April 2025. Following common practice, we evaluate on problems within a recent time window to ensure minimal overlap with training data. The benchmark assesses multiple code-related capabilities including code generation, self-repair (debugging given execution feedback), code execution prediction, and test output prediction. 

\end{itemize}

\subsection{Baselines}
\label{app:baselines}

We compare \ourmodel with a diverse set of baselines spanning supervised learning, reinforcement learning, and distillation-based approaches. These baselines are chosen to reflect different sources of supervision (offline vs. on-policy), as well as different granularities of learning signals (sequence-level vs. token-level).

\begin{itemize}

\item \textbf{Vanilla.}
The base model without any additional task-specific training. This baseline reflects the zero-shot or instruction-tuned capability of the model.

\item \textbf{SFT (Supervised Fine-Tuning).}
A standard supervised learning baseline where the model is trained to maximize the likelihood of reference solutions. This approach corresponds to off-policy imitation learning, where the model learns from fixed expert trajectories. While simple and effective, SFT suffers from exposure bias due to the mismatch between training (on ground-truth prefixes) and inference (on model-generated prefixes).

\item \textbf{GRPO (Group Relative Policy Optimization)}~\citep{grpo}.
A reinforcement learning approach that optimizes the model using relative performance within a group of sampled responses. For each input, multiple candidate outputs are generated, and each is assigned a scalar reward based on task-specific correctness. Instead of learning a separate value function, GRPO computes advantages by normalizing rewards within the group, effectively measuring how much better or worse a response is compared to its peers. The model is then updated using a policy optimization objective with clipping to ensure stable training. While GRPO avoids the need for a learned critic and is relatively efficient, it relies on sparse, sequence-level rewards, assigning the same advantage to all tokens in a response and thus lacking fine-grained token-level supervision.

\item \textbf{OPD (On-Policy Distillation)}~\citep{opd}.
An on-policy distillation method where the model learns from its own generated trajectories while receiving dense token-level supervision from a teacher distribution. At each decoding step, the objective minimizes the divergence between the student and teacher token distributions along the student’s rollout. This approach mitigates exposure bias and provides richer supervision compared to RL, but relies on an external teacher policy.
Following recent findings that combining OPD with GRPO stabilizes training and improves performance~\citep{srpo,sdrlvr,bousselham2025vold,wang2026openclaw,opd,skillsd}, our OPD baseline adopts the joint objective $\mathcal{L} = \mathcal{L}_{\mathrm{GRPO}} + \mathcal{L}_{\mathrm{OPD}}$ (see~\cref{eq:training_object} for the objective of \ourmodel). We additionally report results for naive OPD (without GRPO) in~\cref{app:naive_opd}.

\item \textbf{OPSD$_{\mathrm{gt}}$ (On-Policy Self-Distillation with Ground Truth)}~\citep{opsd}.
A self-distillation variant where the same model plays both teacher and student roles. The student generates responses conditioned only on the input, while the teacher is additionally conditioned on the full ground-truth solution, which is treated as privileged information. The training objective minimizes the token-level divergence between the teacher and student distributions along the student’s trajectory. This setting provides strong supervision but assumes access to complete reference solutions.

\item \textbf{OPSD$_{\mathrm{hint}}$ (On-Policy Self-Distillation with Hints)}~\citep{opsd}.
A more practical variant of OPSD where the privileged information is replaced with a hint derived from the ground-truth solution. Instead of directly exposing the full solution, the teacher conditions on a compressed or partial guidance signal (\ie a hint), which provides weaker but more realistic supervision. This setup evaluates whether partial information can effectively guide self-distillation. Please refer to~\cref{opsd_hint} for the specific prompt in generating hints.

\end{itemize}

\subsection{Implementation Details.}
\label{app:implentation_details}

All experiments are conducted using the VeRL framework~\citep{verl} and Open-AgentRL framework~\citep{demy}. To ensure a fair comparison, all methods share the same base training infrastructure, optimizer, and core hyperparameters unless otherwise noted.
To promote reproducibility and better understanding the details of our implementation, we provide \textbf{an anonymous code link} at the end of the abstract.

\begin{itemize}

    \item  \textbf{Compute Resources.}
All experiments are conducted on a single node with 8 NVIDIA H20 GPUs (96GB memory each). For RL and distillation-based methods, training runs for 1 epoch and typically takes 2-3 days for 0.6B and 1.7B student models (see~\cref{tab:cost}), while training with larger 4B and 14B teacher models requires approximately 5-6 days. 
In contrast, supervised fine-tuning (SFT) is significantly more efficient, with 5 training epochs typically completed within a few hours under the same hardware setup. 
To ensure robustness and statistical reliability, all experiments are repeated over 5 independent runs with different random seeds. These details provide sufficient information about hardware configuration, execution time, and experimental protocol for reproducibility.

    \item  \textbf{Supervised Fine-Tuning (SFT).}
We fine-tune the base Qwen3-0.6B and Qwen3-1.7B models on the 3K SFT corpus for 5 epochs with a global batch size of 128. We use the AdamW optimizer with a learning rate of 5e-5. The maximum sequence length is set to 32,768 tokens with right truncation.

    \item  \textbf{Baselines \& Teacher Models(RL \& Distillation).}
For all RL and distillation-based methods (GRPO, OPD, OPSD$_{\mathrm{gt}}$, OPSD$_{\mathrm{hint}}$) including the teacher models, we adopt a unified training configuration to isolate the effect of each algorithm. Specifically, we use the AdamW optimizer with a learning rate of 1e-6, a training batch size of 64, and a mini-batch size of 16. All RL \& Distillation baselines including \ourmodel are trained based on the SFT checkpoint. The maximum prompt length is set to 2,560 tokens and the maximum response length to 20,480 tokens. We sample 16 responses per prompt during training and 32 during validation. All methods are trained for at most 1 epoch (For the teacher models, we train for at most 2 epochs). The multi-turn agent interaction supports up to 16 tool-call turns. Rollout is performed asynchronously via vLLM with tensor parallelism of 4. For distillation-based methods (OPD and its variants), the teacher model (Qwen3-4B, further optimized with GRPO) provides token-level supervision along the student's on-policy rollouts.

    \item  \textbf{\ourmodel.}
\ourmodel adopts the same training configuration as the baselines described above. The only additional hyperparameters are those governing the step-wise adaptive weighting mechanism: the initial weight $w_1$ is fixed as 1, the numerical stability constant $\epsilon = 10^{-6}$ and the upper-bound offset $\delta = 0.2$ (allowing per-step weights $w_k \leq 1 + \delta$). These parameters control the adaptive distillation strength based on the per-step student--teacher divergence, enabling the model to automatically modulate learning intensity across reasoning steps.

\end{itemize}

\section{Additional Experiments}

\subsection{Training Cost and Practicality Analysis}
\label{app:cost}

A natural concern is whether the step-wise divergence computation in \ourmodel introduces substantial overhead relative to standard OPD. We address this by reporting wall-clock training metrics for all methods. All experiments are conducted on 8$\times$H20 GPUs with identical batch sizes and rollout configurations, trained for 120 steps (0.6B) or 150 steps (1.7B). Key observations are summarized as:

\begin{table}[h]
    \centering
    \caption{
    Training cost comparison. All runs use 8$\times$H20 96GB GPUs with identical configurations.}
    \vspace{0.4 em}
    \label{tab:cost}

    \begingroup
    \setlength{\tabcolsep}{4pt}
    \renewcommand{\arraystretch}{1.08}
    \normalsize

    \begin{adjustbox}{max width=0.76\textwidth}
    \begin{tabular}{llccc}
        \toprule
        \multirow{2}{*}{\textbf{Params}} 
        & \multirow{2}{*}{\textbf{Method}} 
        & \textbf{Avg Time/Step} 
        & \textbf{Peak Memory} 
        & \textbf{Total Time} \\
        & & (s) & (GB) & (h) \\

        \midrule
        \multirow{3}{*}{\textbf{0.6B}}
        & GRPO 
        & 680.6 & 78.5 & 22.7 \\
        & OPD 
        & 1090.5 & 87.7 & 36.4 \\
        & \ourmodel 
        & 1052.3 & 87.6 & 35.1 \\

        \midrule
        \multirow{3}{*}{\textbf{1.7B}}
        & GRPO 
        & 784.6 & 81.7 & 32.7 \\
        & OPD 
        & 1053.6 & 88.2 & 43.9 \\
        & \ourmodel 
        & 1105.4 & 88.7 & 46.1 \\

        \bottomrule
    \end{tabular}%
    \end{adjustbox}

    \endgroup
    \vspace{-10pt}
\end{table}

\paragraph{SOD vs.\ OPD: negligible algorithmic overhead.}
The step-wise divergence $d_k$ (\Cref{eq:step_div}) and adaptive weight $w_k$ (\cref{eq:adaptive_weight}) are computed from the student and teacher log-probabilities already available during the OPD forward pass. The additional computation is merely per-step averaging and $O(K)$ scalar multiplications ($K{\approx}3$--$7$), requiring no extra forward pass. Peak memory is equivalent ($<0.5$~GB difference). Notably, for the 0.6B model, \ourmodel is actually \textbf{3.5\% faster} per step than OPD (1052.3s vs.\ 1090.5s), because the adaptive reweighting suppresses learning from erroneous tool-call patterns, resulting in more efficient rollouts with fewer failed retries and shorter responses.
For the 1.7B model, \ourmodel shows a modest $+4.9\%$ overhead (1105.4s vs.\ 1053.6s), which is entirely acceptable given the significant performance gains (+18.50\% improvement at average) over OPD.

\paragraph{Why GRPO has lower per-step time.}
GRPO exhibits $\sim$1.4--1.5$\times$ lower per-step time than OPD-based methods. However, this is largely because GRPO's training collapses in later stages: as shown in~\cref{rq4}, the GRPO-trained student progressively loses the ability to perform effective tool calls, resulting in drastically shorter responses and fewer sandbox interactions in later training steps. This artificially deflates the average per-step time. In contrast, OPD and \ourmodel maintain long, structured reasoning trajectories with active tool usage throughout training, which naturally requires more generation time per step.

\subsection{Quantitative Analysis of Distillation Pattern Distribution}
\label{app:futher_casestudy}

\begin{figure*}[h]
\centering
\includegraphics[width=1.0\linewidth]{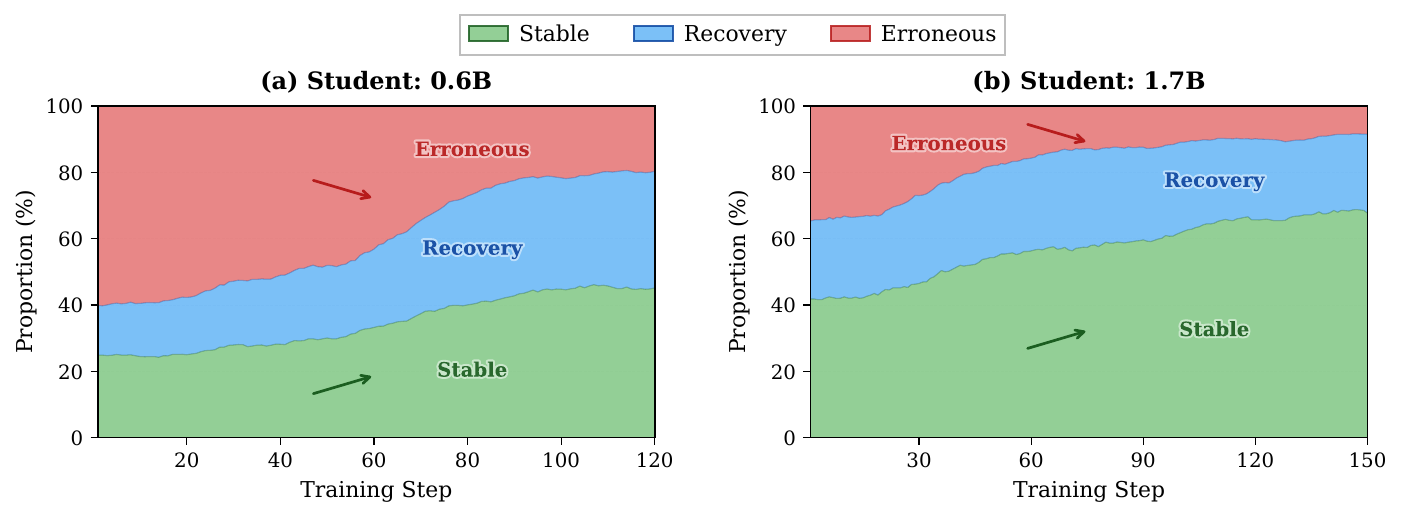}
\caption{Distribution of three distillation patterns over training steps. At each step, all rollout trajectories are classified into Stable, Recovery, or Erroneous based on their adaptive weight dynamics, and the proportion of each pattern is reported (smoothed with a 9-step moving average).}
\label{fig:pattern_dist}
\end{figure*}

To move beyond qualitative illustration and quantify how the three distillation patterns evolve during training, we conduct a systematic classification of all multi-step trajectories at each training step. Specifically, for every rollout sample at global step $t$, we record the full sequence of adaptive weights $\{w_k\}_{k=1}^{K}$ computed by \ourmodel during that trajectory. We then classify each trajectory into one of three dominant patterns based on the overall shape of the $w_k$ sequence:

\begin{itemize}[leftmargin=1.5em, itemsep=2pt]
    \item \textbf{Stable}: Adaptive weights remain consistently high throughout the trajectory, indicating that the student stays well-aligned with the teacher across all reasoning steps.
    \item \textbf{Erroneous}: Weights are progressively suppressed and remain low by the final step, indicating persistent divergence that the student fails to correct.
    \item \textbf{Recovery}: Weights drop during intermediate steps (reflecting temporary divergence) but recover by the final step, indicating that the student successfully re-aligns after an initial deviation.
\end{itemize}

\noindent We note that a fourth logical category exists, \ie trajectories that begin aligned but \emph{degrade} toward the end (\ie early weights are high but the final weight is low). In practice, we find this pattern accounts for a negligible fraction ($<$1\%) of all trajectories, as persistent degradation without any preceding divergence signal is rare under \ourmodel's cumulative reweighting mechanism. We therefore merge these cases into the erroneous category for simplicity.
At each training step, we classify all rollout samples and compute the proportion of each pattern. Results are smoothed with a 9-step moving average for visualization clarity. Figure~\ref{fig:pattern_dist} reports the distribution over the training trajectory for both model scales.

\paragraph{Key observations.}
\textbf{(1) Stable proportion increases over training.} For both 0.6B and 1.7B, the fraction of stable trajectories grows steadily, confirming that \ourmodel progressively improves student-teacher alignment as training proceeds.
\textbf{(2) Erroneous proportion decreases consistently.} The proportion of erroneous trajectories declines substantially over training. This indicates that, as training progresses under \ourmodel, the student model becomes increasingly capable: when encountering difficult reasoning steps, it no longer falls into consecutive errors but instead recovers or avoids mistakes altogether.
\textbf{(3) Recovery grows and becomes a persistent component.} Notably, the recovery proportion is initially low (especially for the 0.6B model) and \emph{increases} over training before stabilizing at a substantial level. This reveals a meaningful progression: early in training, the student lacks the capacity to recover from divergence; once it deviates, it tends to remain erroneous. As training proceeds, the student develops the ability to self-correct after intermediate missteps, converting what would have been erroneous trajectories into recovery ones. Even at the end of training, recovery trajectories account for a significant proportion of all samples, indicating that \ourmodel's adaptive reweighting mechanism remains \emph{continuously active}.
\textbf{(4) Model capacity affects convergence speed and pattern composition.} The 1.7B model reaches a higher stable proportion faster and drives erroneous to a lower level by the end of training, while maintaining a moderate recovery proportion throughout. In contrast, the 0.6B model starts with a much higher erroneous fraction and lower recovery fraction, reflecting its initially limited ability to self-correct. As training progresses, the 0.6B model gradually develops recovery capability, but converges more slowly overall. This suggests that \ourmodel's protective suppression of erroneous signals is particularly critical for smaller models during early training.

\subsection{The Performance of Naive On-policy Distillation}
\label{app:naive_opd}

In the main experiments, our OPD baseline combines the OPD loss with GRPO (\ie $\mathcal{L} = \mathcal{L}_{\mathrm{GRPO}} + \mathcal{L}_{\mathrm{OPD}}$) as illustrated in~\cref{app:baselines}. Here we report results for \textit{naive OPD}, which uses only the token-level distillation loss without any RL reward signal.

\begin{table}[h]
    \centering
    \caption{Comparison of naive OPD vs.\ OPD on the 1.7B student model.}
    \vspace{0.4em}
    \label{tab:naive_opd}
    \begingroup
    \setlength{\tabcolsep}{4pt}
    \renewcommand{\arraystretch}{1.08}
    \begin{adjustbox}{max width=0.85\textwidth}
    \begin{tabular}{lccccc}
        \toprule
        \multirow{2}{*}{\textbf{Method}}
        & \multicolumn{2}{c}{\textbf{Math}}
        & \textbf{Science}
        & \textbf{Code}
        & \multirow{2}{*}{\textbf{Average}} \\
        \cmidrule(lr){2-3} \cmidrule(lr){4-4} \cmidrule(lr){5-5}
        & AIME 2024 & AIME 2025 & GPQA & LiveCodeBench & \\
        \midrule
        naive OPD
        & 40.54\phantom{.}\scalebox{0.72}{$\pm$1.05} & 35.83\phantom{.}\scalebox{0.72}{$\pm$1.11} & 28.96\phantom{.}\scalebox{0.72}{$\pm$0.58} & 29.81\phantom{.}\scalebox{0.72}{$\pm$0.74} & 33.79 \\
        OPD 
        & 43.86\phantom{.}\scalebox{0.72}{$\pm$1.23} & 37.04\phantom{.}\scalebox{0.72}{$\pm$1.31} & 31.73\phantom{.}\scalebox{0.72}{$\pm$0.55} & 32.45\phantom{.}\scalebox{0.72}{$\pm$0.93} & 36.27 \\
        \midrule
        \ourmodel (w/o GRPO)
        & 48.87\phantom{.}\scalebox{0.72}{$\pm$0.93} & 39.73\phantom{.}\scalebox{0.72}{$\pm$1.02} & 35.89\phantom{.}\scalebox{0.72}{$\pm$0.62} & 38.62\phantom{.}\scalebox{0.72}{$\pm$0.71} & 40.78 \\
        \textbf{\ourmodel}
        & \textbf{50.83}\phantom{.}\scalebox{0.72}{$\pm$1.15} & \textbf{41.72}\phantom{.}\scalebox{0.72}{$\pm$1.24} & \textbf{38.72}\phantom{.}\scalebox{0.72}{$\pm$0.73} & \textbf{40.63}\phantom{.}\scalebox{0.72}{$\pm$0.91} & \textbf{42.98} \\
        \bottomrule
    \end{tabular}
    \end{adjustbox}
    \endgroup
\end{table}

As shown in~\cref{tab:naive_opd}, naive OPD underperforms the OPD baseline by 2.48 points on average (33.79 vs.\ 36.27). Without the outcome-level reward signal from GRPO, the student relies entirely on dense teacher supervision for learning. However, as analyzed in~\cref{sec:failure_uniform_opd},  OPD applies uniform distillation across all steps including those corrupted by tool-induced state drift. In the absence of GRPO's reward-driven exploration, the model lacks an independent signal to distinguish successful from failed trajectories, making it more susceptible to the compounding error problem.

More notably, comparing the two pure distillation variants, naive OPD (33.79) vs.\ \ourmodel w/o GRPO (40.78), reveals that our step-wise reweighting alone accounts for a +6.99 (+20.69\%) improvement without any RL reward signal. This confirms that the core benefit of \ourmodel stems from the adaptive step-wise reweighting mechanism rather than from the RL component. GRPO provides complementary gains (+2.20 on top of step-wise OPD), but the reweighting mechanism remains the primary driver of performance improvement. These results also suggest that step-wise reweighting and GRPO address orthogonal aspects of the training challenge: the former stabilizes dense supervision under state drift, while the latter provides sparse outcome-level exploration signals.

\section{Proofs for~\cref{sec:failure_uniform_opd}}
\label{app:proof_opd_failure}

We adopt the notation established in Section~\ref{sec:failure_uniform_opd}. For notational convenience in the derivations below, we write $p_t(\cdot) = \pi_\theta(\cdot \mid y_{<t})$ and $q_t(\cdot) = \pi_{\mathrm{teacher}}(\cdot \mid y_{<t})$ for the student and teacher next-token distributions conditioned on the student-generated prefix $y_{<t}$. The token-level OPD loss is $\ell_t = \log p_t(y_t) - \log q_t(y_t)$ with $y_t \sim p_t$. The step-level mismatch $\Delta_k$, the teacher-supported region $S_t^\epsilon$, and the overlap $\rho_t$ are as defined in the main text.

\subsection{Proof of Proposition~\ref{prop:discontinuous_drift}}

\begin{proof}
We analyze how the step-level mismatch $\Delta_k = \frac{1}{|\mathcal{I}_k|}\sum_{t \in \mathcal{I}_k} D_{\mathrm{KL}}(p_t \| q_t)$ evolves between consecutive steps in two settings.

\paragraph{Text-only case (gradual drift).}
In text-only generation, the prefix at position $t$ differs from position $t-1$ by exactly one student-sampled token. Define the per-token distributional shift as:
\begin{equation}
    \eta \triangleq \max_t \; \mathrm{TV}(p_t, p_{t-1}) = \max_t \; \frac{1}{2}\sum_{v \in \mathcal{V}}\left|p_t(v) - p_{t-1}(v)\right|,
\end{equation}
\ie the maximum total variation (TV) distance between consecutive output distributions. For well-trained transformers, $\eta$ is small because the self-attention mechanism distributes influence across many context positions, so appending a single token has limited impact on the output distribution. Since consecutive steps share a boundary of one token, the step-level mismatch changes by at most $O(\eta)$ per step transition.

\paragraph{TIR case (discontinuous drift).}
In TIR, the transition from step $k$ to step $k+1$ involves appending an entire tool observation $o_k = (o_k^1, \ldots, o_k^m)$ of length $m = |o_k|$ to the prefix:
\begin{equation}
    y_{<t_{k+1}^{\mathrm{start}}} = y_{<t_k^{\mathrm{end}}} \;\oplus\; o_k,
\end{equation}
where $\oplus$ denotes concatenation. Crucially, $o_k$ is produced by the external environment, not by either model. Let $\eta_{\mathrm{tool}}$ denote the average per-token distributional shift induced by observation tokens, measured by the TV distance between the model's output distribution before and after conditioning on each successive observation token.

\paragraph{Single-step bound.}
By the triangle inequality applied to the $m$ intermediate distributions between $p_{t_k^{\mathrm{end}}}$ and $p_{t_{k+1}^{\mathrm{start}}}$, the total variation shift satisfies:
\begin{equation}
    \mathrm{TV}\!\left(p_{t_{k+1}^{\mathrm{start}}},\; p_{t_k^{\mathrm{end}}}\right) \;\le\; m \cdot \eta_{\mathrm{tool}}.
\end{equation}
A similar bound holds for the teacher distribution. By Pinsker's inequality ($D_{\mathrm{KL}}(P\|Q) \ge 2\,\mathrm{TV}(P,Q)^2$), the change in step-level KL divergence is lower-bounded by the squared TV shift. Since tool observations typically span tens to hundreds of tokens ($m \gg 1$) and introduce content not generated by either model, we have $\eta_{\mathrm{tool}} > \eta$ in general, yielding:
\begin{equation}
    \Delta_{k+1} - \Delta_k \;=\; \Omega(m \cdot \eta_{\mathrm{tool}}) \;\ge\; O(\eta).
\end{equation}
Note that even a single erroneous tool call already introduces a substantially larger divergence jump than text-only drift ($\Omega(m \cdot \eta_{\mathrm{tool}})$ vs.\ $O(\eta)$), since the observation length $m \gg 1$ and the out-of-distribution content shifts conditioning more aggressively. Nevertheless, the teacher, having encountered some error patterns during training, can still provide partially useful supervision after an isolated failure. The critical issue arises from the \emph{cascading} nature of errors in weaker student models, as we analyze next.

\paragraph{Cascading errors and super-linear compounding.}
Weaker student models, precisely the targets of OPD, are prone to producing consecutive erroneous tool calls. A weaker student is more likely to generate an incorrect tool invocation at step $k$; conditioned on the resulting erroneous observation, its subsequent reasoning is further degraded, increasing the probability of another failure at step $k+1$. This creates a cascading failure pattern where each error compounds upon the previous ones.

The key insight is that while the teacher may reasonably handle a single isolated error in the context (having potentially seen similar error messages during training), the \emph{joint} occurrence of multiple consecutive failures creates a conditioning context that is combinatorially unlikely under the teacher's training distribution. If each individual error context has marginal probability $p_{\mathrm{err}}$ under the teacher's training distribution, the joint context of $j$ consecutive errors has probability at most $p_{\mathrm{err}}^j$, an exponential decay in familiarity. It is this accumulated, multi-error prefix that drives the teacher's conditional distribution far from calibration, making $\eta_{\mathrm{tool}}^{(i)}$ grow with $i$.

Formally, when the student makes consecutive erroneous tool calls across steps $k, k+1, \ldots, k+j-1$, the cumulative divergence satisfies:
\begin{equation}
    \Delta_{k+j} - \Delta_k \;=\; \Omega\!\left(\sum_{i=0}^{j-1} m_{k+i} \cdot \eta_{\mathrm{tool}}^{(k+i)}\right),
\end{equation}
where $m_{k+i}$ is the length of the $i$-th erroneous observation and $\eta_{\mathrm{tool}}^{(k+i)}$ is the corresponding per-token shift. Crucially, later errors induce progressively larger per-token shifts ($\eta_{\mathrm{tool}}^{(k+i+1)} \ge \eta_{\mathrm{tool}}^{(k+i)}$) because the prefix is already corrupted by prior errors, where the teacher has never been calibrated on such multi-error contexts, making its predictions increasingly unreliable. This yields super-linear growth of divergence with the number of consecutive tool failures, which is empirically confirmed in~\cref{motivation}(a) where the student-teacher divergence accelerates as errors accumulate.
\end{proof}

\subsection{Proof of Proposition~\ref{prop:variance_explosion}}

\begin{proof}
We prove the second-moment lower bound and then derive the SNR degradation.

\paragraph{Step 1: Decomposition of the second moment.}
The second moment of the OPD loss $\ell_t = \log p_t(y_t) - \log q_t(y_t)$ under $y_t \sim p_t$ decomposes over the vocabulary:
\begin{equation}
    \mathbb{E}_{y_t \sim p_t}[\ell_t^2]
    \;=\;
    \sum_{v \in \mathcal{V}} p_t(v) \left(\log \frac{p_t(v)}{q_t(v)}\right)^2.
\end{equation}

\paragraph{Step 2: Restricting to the low-overlap region.}
We partition the vocabulary into the teacher-supported region $S_t^\epsilon = \{v : q_t(v) \ge \epsilon\}$ and its complement $\bar{S}_t^\epsilon = \{v : q_t(v) < \epsilon\}$. By the overlap condition $\rho_t \le \rho$, the student places mass at least $1 - \rho$ on $\bar{S}_t^\epsilon$:
\begin{equation}
    \sum_{v \in \bar{S}_t^\epsilon} p_t(v) \;\ge\; 1 - \rho_t \;\ge\; 1 - \rho.
\end{equation}
Restricting the sum to $\bar{S}_t^\epsilon$ gives a lower bound:
\begin{equation}
    \mathbb{E}[\ell_t^2]
    \;\ge\;
    \sum_{v \in \bar{S}_t^\epsilon} p_t(v) \left(\log \frac{p_t(v)}{q_t(v)}\right)^2.
\end{equation}

\paragraph{Step 3: Applying Jensen's inequality.}
For any $v \in \bar{S}_t^\epsilon$, since $q_t(v) < \epsilon$, we have:
\begin{equation}
    \log \frac{p_t(v)}{q_t(v)} \;>\; \log p_t(v) + \log \frac{1}{\epsilon}.
\end{equation}
Define the conditional distribution $\tilde{p}(v) = p_t(v)/(1-\rho_t)$ over $\bar{S}_t^\epsilon$. By the convexity of $x \mapsto x^2$ and Jensen's inequality:
\begin{align}
    \mathbb{E}[\ell_t^2]
    &\;\ge\;
    (1-\rho_t) \sum_{v \in \bar{S}_t^\epsilon} \tilde{p}(v) \left(\log \frac{p_t(v)}{q_t(v)}\right)^2 \nonumber \\
    &\;\ge\;
    (1-\rho_t) \left(
    \sum_{v \in \bar{S}_t^\epsilon} \tilde{p}(v) \log \frac{p_t(v)}{q_t(v)}
    \right)^2 \nonumber \\
    &\;\ge\;
    (1-\rho_t) \left(
    \log \frac{1}{\epsilon} + \sum_{v \in \bar{S}_t^\epsilon} \tilde{p}(v) \log p_t(v)
    \right)^2.
\end{align}

\paragraph{Step 4: Bounding the entropy term.}
The term $\sum_{v \in \bar{S}_t^\epsilon} \tilde{p}(v) \log p_t(v) = -H_{\bar{S}}$, where $H_{\bar{S}}$ is the conditional entropy of the student restricted to $\bar{S}_t^\epsilon$, satisfying $0 \le H_{\bar{S}} \le \log|\mathcal{V}|$. For sufficiently small $\epsilon$ satisfying $\epsilon < 1/|\mathcal{V}|$ (a natural condition since the teacher threshold should be below the uniform baseline), we have $\log(1/\epsilon) > \log|\mathcal{V}| \ge H_{\bar{S}}$, so the $\log(1/\epsilon)$ term dominates:
\begin{equation}
    \mathbb{E}[\ell_t^2]
    \;\ge\;
    (1-\rho)\left(\log \frac{1}{\epsilon} - H_{\bar{S}}\right)^2
    \;\ge\;
    (1-\rho)\,\log^2\!\frac{1}{\epsilon} \cdot \left(1 - \frac{\log|\mathcal{V}|}{\log(1/\epsilon)}\right)^2,
\end{equation}
which grows as $\Omega(\log^2(1/\epsilon))$ when $\rho_t \to 0$.

\paragraph{Step 5: SNR degradation.}
Let $g_t = \ell_t \nabla_\theta \log \pi_\theta(y_t \mid y_{<t})$ be the token-level OPD gradient estimator. The expected gradient is $\mathbb{E}[g_t] = \nabla_\theta D_{\mathrm{KL}}(p_t \| q_t)$, which is bounded for any fixed distribution pair with finite KL divergence; let $C = \|\mathbb{E}[g_t]\|$.

Define $c_{\min} \triangleq \min_v \|\nabla_\theta \log \pi_\theta(v \mid y_{<t})\|^2 > 0$, which is positive for any non-degenerate parameterization. The variance of the gradient estimator satisfies:
\begin{equation}
    \mathrm{Var}[\|g_t\|]
    \;\ge\;
    \mathbb{E}[\ell_t^2] \cdot c_{\min} - C^2.
\end{equation}
Combining with the second-moment lower bound from Step 4:
\begin{equation}
    \mathrm{SNR}(g_t)
    \;=\;
    \frac{\|\mathbb{E}[g_t]\|}{\sqrt{\mathrm{Var}[g_t]}}
    \;\le\;
    \frac{C}{\sqrt{(1-\rho)\log^2(1/\epsilon) \cdot c_{\min} - C^2}}.
\end{equation}
As $\rho \to 0$, the denominator grows as $O(\log(1/\epsilon))$ while the numerator $C$ remains constant, giving $\mathrm{SNR}(g_t) \to 0$.

\paragraph{Interpretation.} When the student drifts far from the teacher-supported region ($\rho_t \approx 0$), most sampled tokens fall in $\bar{S}_t^\epsilon$ where the teacher assigns near-zero probability. The OPD loss for these tokens has large magnitude (due to $-\log q_t(y_t) > \log(1/\epsilon)$) but is uninformative. It reflects the teacher's inability to model these out-of-distribution states rather than a meaningful learning signal. Meanwhile, informative signals from $S_t^\epsilon$ (where teacher guidance is calibrated) are sampled only with probability $\rho_t$, making them increasingly rare. The gradient is thus dominated by high-variance, low-information contributions, rendering optimization unstable.
\end{proof}

\subsection{Discussion: Failure Cascade and Design Implications}
\label{app:discussion_objective}

The two propositions together characterize a failure cascade specific to TIR:
\begin{enumerate}
    \item \textbf{Initial perturbation}: An erroneous tool call returns a corrupted observation (\eg a runtime error, incorrect output, or timeout message). This already introduces a divergence jump substantially larger than text-only drift ($\Omega(m \cdot \eta_{\mathrm{tool}})$ vs.\ $O(\eta)$), though the teacher, having encountered some error patterns during pretraining, can still provide partially useful supervision after an isolated failure.
    
    \item \textbf{Cascading accumulation}: Weaker student models, precisely the targets of OPD, are prone to making consecutive errors. Each subsequent erroneous tool call further corrupts the prefix, and the \emph{joint} pattern of multiple consecutive failures becomes exponentially unlikely under the teacher's training distribution ($\sim p_{\mathrm{err}}^j$ for $j$ consecutive errors). It is this accumulated multi-error context, rather than any single error, that drives the teacher's conditional distribution far from calibration. The divergence thus compounds super-linearly (Proposition~\ref{prop:discontinuous_drift}), as empirically demonstrated by the accelerating divergence curve in~\cref{motivation}(a).
    
    \item \textbf{Supervision breakdown}: In the resulting low-overlap states ($\rho_t \approx 0$) caused by accumulated consecutive errors, the OPD gradient estimator suffers variance explosion and SNR degradation (Proposition~\ref{prop:variance_explosion}). Updates become dominated by uninformative, high-magnitude contributions from tokens where the teacher provides no meaningful guidance.~\cref{motivation}(b) confirms this empirically: the teacher's conditional entropy becomes both elevated and highly variable in steps following accumulated tool errors.
    
    \item \textbf{Amplification by uniform aggregation}: OPD sums token-level losses across all steps with equal weight, treating well-aligned early steps (high $\rho_t$, reliable teacher guidance) identically to corrupted post-error steps (low $\rho_t$, unreliable guidance). When cumulative tool errors cause a substantial portion of later steps to have low overlap, the aggregate gradient is systematically biased by these high-variance contributions.
\end{enumerate}

This analysis implies that a stabilizing objective should modulate distillation strength at the step level: preserving full-strength dense supervision in well-aligned steps while attenuating the signal when the estimated student-teacher divergence indicates teacher miscalibration. This is the design principle underlying our adaptive step-wise reweighting mechanism.

\subsection{Variance Reduction under Step-wise Reweighting}
\label{app:variance_reduction}

We now show that the step-wise reweighting mechanism in Eq.~\eqref{eq:adaptive_weight} provably addresses the gradient SNR degradation identified in Proposition~\ref{prop:variance_explosion}.

\begin{proposition}[Bounded variance under adaptive reweighting]
\label{prop:variance_reduction}
Under the step-wise weighted OPD objective (Eq.~\ref{eq:step_opd}), when the divergence increases monotonically across steps ($d_1 \le d_2 \le \cdots \le d_k$), the weighted gradient contribution from step $k$ satisfies:
\begin{equation}
    \mathbb{E}[w_k^2 \cdot \ell_t^2] \;\le\; \frac{(d_1 + \epsilon)^2}{(d_k + \epsilon)^2} \cdot \mathbb{E}[\ell_t^2],
\end{equation}
where $\mathbb{E}[\ell_t^2]$ is the unweighted second moment from Proposition~\ref{prop:variance_explosion}. Consequently, even when $\rho_t \to 0$ causes $\mathbb{E}[\ell_t^2] \to \infty$, the weighted contribution remains bounded whenever $d_k$ grows proportionally to the divergence.
\end{proposition}

\begin{proof}
By the definition of $w_k$ in Eq.~\eqref{eq:adaptive_weight}:
\begin{equation}
    w_k = \min\!\left(\prod_{u=1}^{k-1} \frac{d_u + \epsilon}{d_{u+1} + \epsilon},\; 1+\delta\right).
\end{equation}
Under monotonically increasing divergence ($d_1 \le d_2 \le \cdots \le d_k$), each ratio $\frac{d_u + \epsilon}{d_{u+1} + \epsilon} \le 1$, and the product telescopes:
\begin{equation}
    w_k \;\le\; \prod_{u=1}^{k-1} \frac{d_u + \epsilon}{d_{u+1} + \epsilon} \;=\; \frac{d_1 + \epsilon}{d_k + \epsilon}.
\end{equation}
The weighted second moment of the OPD loss at token $t \in \mathcal{I}_k$ is therefore:
\begin{equation}
    \mathbb{E}[w_k^2 \cdot \ell_t^2] \;\le\; \left(\frac{d_1 + \epsilon}{d_k + \epsilon}\right)^2 \cdot \mathbb{E}[\ell_t^2].
\end{equation}

From Proposition~\ref{prop:variance_explosion}, when the student drifts into low-overlap states, $\mathbb{E}[\ell_t^2] \ge (1-\rho)\log^2(1/\epsilon_{\mathrm{teacher}})$. However, such drift also implies that $d_k$ increases (since $d_k$ is monotonically related to $\Delta_k$, as shown in Appendix~\ref{app:dk_proxy}). Specifically, when $\rho_t \to 0$, the per-token log-probability gaps grow, causing $d_k \gg d_1$. The weighted contribution thus satisfies:
\begin{equation}
    \mathbb{E}[w_k^2 \cdot \ell_t^2] \;\le\; \frac{(d_1 + \epsilon)^2}{(d_k + \epsilon)^2} \cdot (1-\rho)\log^2(1/\epsilon_{\mathrm{teacher}}).
\end{equation}

The key insight is that the numerator $(d_1 + \epsilon)^2$ is bounded (determined by the initial step's divergence), while the denominator $(d_k + \epsilon)^2$ grows with the accumulated drift. This creates an automatic variance suppression: the more the student diverges (larger $d_k$, larger $\mathbb{E}[\ell_t^2]$), the more aggressively the weight $w_k$ attenuates the contribution, preventing the variance explosion that afflicts OPD.

\paragraph{SNR recovery.}
For the weighted gradient estimator $\tilde{g}_t = w_k \cdot \ell_t \cdot \nabla_\theta \log \pi_\theta(y_t \mid y_{<t})$, the signal (expected gradient from well-aligned early steps where $w_k \approx 1$) remains intact, while the noise from high-divergence later steps is suppressed by the factor $(d_1 + \epsilon)^2/(d_k + \epsilon)^2$. This restores a positive SNR for the aggregate gradient, in contrast to OPD where the SNR degrades to zero (Proposition~\ref{prop:variance_explosion}).
\end{proof}

\paragraph{Stability under recovery.}
The above analysis addresses the erroneous case where divergence increases monotonically. In the recovery case ($d_{u+1} < d_u$ for some $u$), the weight $w_k$ may exceed 1, potentially amplifying the gradient contribution. However, the hard upper bound $w_k \le 1 + \delta$ in Eq.~\eqref{eq:adaptive_weight} ensures that the weighted second moment is always bounded by $(1+\delta)^2 \cdot \mathbb{E}[\ell_t^2]$. Since recovery steps by definition have decreasing $d_k$ (implying the student is returning to the teacher-supported region, \ie $\rho_t$ is increasing), the unweighted $\mathbb{E}[\ell_t^2]$ itself is decreasing in these steps. The combination of bounded amplification and decreasing base variance ensures optimization stability.

\paragraph{Interpretation.}
This result formalizes the intuition behind our method: OPD fails because it applies uniform weight to all steps, allowing high-variance contributions from corrupted post-error steps to dominate the gradient. Our step-wise reweighting automatically detects these corrupted steps (via increasing $d_k$) and attenuates their influence proportionally, ensuring that the gradient signal remains dominated by informative contributions from well-aligned steps. The upper bound $1+\delta$ further ensures that the recovery mechanism does not over-amplify any single step, maintaining optimization stability throughout training.

\subsection{Justification of $d_k$ as a Proxy for $\Delta_k$}
\label{app:dk_proxy}

Computing the full KL divergence $\Delta_k$ requires evaluating the teacher's output distribution over the entire vocabulary at each token position, which introduces substantial overhead. In contrast, $d_k$ (Eq.~\ref{eq:step_div}) only requires the teacher's log-probability on the student-sampled token $y_t$, a quantity already computed in the standard OPD forward pass.

\paragraph{Monotonic consistency.}
By Jensen's inequality:
\begin{equation}
    d_k = \frac{1}{|\mathcal{I}_k|} \sum_{t \in \mathcal{I}_k} \left| \log \frac{\pi_\theta(y_t \mid y_{<t})}{\pi_{\mathrm{teacher}}(y_t \mid y_{<t})} \right|
    \;\ge\;
    \left| \frac{1}{|\mathcal{I}_k|} \sum_{t \in \mathcal{I}_k} \log \frac{\pi_\theta(y_t \mid y_{<t})}{\pi_{\mathrm{teacher}}(y_t \mid y_{<t})} \right|,
\end{equation}
where the right-hand side is the absolute value of a single-sample Monte Carlo estimate of $\Delta_k$. When the student drifts further from the teacher (increasing $\Delta_k$), the per-token log-probability gaps increase in expectation, yielding a larger $d_k$. This monotonic relationship ensures that the ratios $d_u / d_{u+1}$ used in Eq.~\eqref{eq:adaptive_weight} correctly reflect the relative trend of divergence across steps.

\paragraph{Sufficiency of ordering.}
The weight formula in Eq.~\eqref{eq:adaptive_weight} is a product of ratios $\frac{d_u + \epsilon}{d_{u+1} + \epsilon}$, depending only on the relative magnitudes between consecutive steps. Any monotone transformation of $\Delta_k$ preserves these ratios' direction (above or below one), producing equivalent attenuation and recovery behavior. Since $d_k$ maintains the ordering of $\Delta_k$, it is a sufficient statistic for our reweighting mechanism.

\section{Prompt Template}

\subsection{Datasets Prompt Template}

\label{dataset_prompt}

Our dataset contains various question types, including mathematics, programming, and scientific problems, each with different answer formats. To ensure consistent model output and facilitate simultaneous reasoning and tool usage, we have developed specialized prompts for each task type. Below, we present these prompts following~\citet{demy}.

\begin{tcolorbox}[
colframe=gray!30,
colback=gray!5,
coltitle=black!100,
title=\textit{\textbf{Prompt Template for verifiable Math Problems}}
]
\setstretch{0.93}

\rule{1.0\textwidth}{0.4pt}
\hrule
\vspace{0.3ex}
\hrule

\begin{center}
    \textbf{\large Math Problems}
\end{center}
\vspace{-0.8em}
\rule{1.0\textwidth}{0.4pt}

Analyze and solve the following [math/science domain] problem step by step.

\rule{1.0\textwidth}{0.4pt}

\textbf{Problem:} \textcolor{violet!90}{\textbf{[Insert problem text here]}}

\rule{1.0\textwidth}{0.4pt}

\textbf{Hint:} \textcolor{blue!80!black}{\textbf{The tool could be used}} for more precise and efficient calculations and could help you to verify your result before you reach the final answer.

\rule{1.0\textwidth}{0.4pt}

\textbf{Note:} You should first analyze the problem and form a high-level solution strategy, then utilize the tools to help you solve the problem.

\rule{1.0\textwidth}{0.4pt}

\textbf{Answer Format:} Do not put units of the final answer inside \textbackslash{}boxed\{\}. The content of \textbackslash{}boxed\{\} should be the numerical value of the final answer only, without any units.

Remember once you make sure the current answer is your final answer, do not call the tools again and directly output the final answer in the following text format, the answer format must be: \textcolor{red!50!black}{\textbf{\texttt{\textbackslash{}boxed\{'The final answer goes here.'\}}}}.

\end{tcolorbox}

\begin{tcolorbox}[
colframe=gray!30,
colback=gray!5,
coltitle=black!100,
title=\textit{\textbf{Prompt Template for Scientific QA Problems}}
]
\setstretch{0.93}

\rule{1.0\textwidth}{0.4pt}
\hrule
\vspace{0.3ex}
\hrule

\begin{center}
    \textbf{\large Scientific QA Problems}
\end{center}
\vspace{-0.8em}
\rule{1.0\textwidth}{0.4pt}

Analyze and solve the following [science domain] problem step by step.

\rule{1.0\textwidth}{0.4pt}

\textbf{Problem:} \textcolor{violet!90}{\textbf{[Insert problem text here]}}

\rule{1.0\textwidth}{0.4pt}

\textbf{Hint:} \textcolor{blue!80!black}{\textbf{The tool could be used}} for more precise and efficient calculations and could help you to verify your result before you reach the final answer.

\rule{1.0\textwidth}{0.4pt}

\textbf{Note:} You should first analyze the problem and form a high-level solution strategy, then utilize the tools to help you solve the problem.

\rule{1.0\textwidth}{0.4pt}

\textbf{Answer Format:} Remember once you make sure the current answer is your final answer, do not call the tools again and directly output the final answer in the following text format, the answer format must be: \textcolor{red!50!black}{\textbf{\texttt{\textbackslash{}boxed\{'The final answer goes here.'\}}}}. You need to put the final uppercase letter option of this problem into \textbackslash{}boxed\{\}.

\end{tcolorbox}

\begin{tcolorbox}[
colframe=gray!30,
colback=gray!5,
coltitle=black!100,
title=\textit{\textbf{Prompt Template for Code Problems}}
]
\setstretch{0.93}

\rule{1.0\textwidth}{0.4pt}
\hrule
\vspace{0.3ex}
\hrule

\begin{center}
    \textbf{\large Code Problems}
\end{center}
\vspace{-0.8em}
\rule{1.0\textwidth}{0.4pt}

You will be given a question (problem specification) and will generate a correct Python program that matches the specification and passes all tests.

\rule{1.0\textwidth}{0.4pt}

\textbf{Problem:} \textcolor{violet!90}{\textbf{[Insert problem text here]}}

\rule{1.0\textwidth}{0.4pt}

\textbf{Public Examples:} Here are some input and output examples of the expected code:

\textcolor{green!50!black}{\textbf{Input:}} [sample inputs]

\textcolor{green!50!black}{\textbf{Output:}} [sample outputs]

\rule{1.0\textwidth}{0.4pt}

\textbf{Note:} You should first analyze the problem and form a high-level solution strategy, then utilize the tools to help you solve the problem.

\rule{1.0\textwidth}{0.4pt}

\textbf{Instruction:} Read the inputs from stdin, solve the problem, and write the answer to stdout (do not directly test on the sample inputs). Enclose your code within the delimiters shown below. Ensure that when the Python program runs, it correctly reads inputs, executes the algorithm, and writes output to stdout.

\rule{1.0\textwidth}{0.4pt}

\textbf{Submit:} Before submitting your code, you can utilize tools to check its correctness. Once you make sure the current code is correct, do not call the tools again and submit your code within the following Python code block:

\begin{lstlisting}[language=Python]
# YOUR CODE HERE
\end{lstlisting}

\end{tcolorbox}

\subsection{Baseline Prompt Template}
\label{opsd_hint}
For the OPSD$_{\mathrm{hint}}$ baseline, we construct partial guidance signals by prompting a language model to distill the ground-truth solution into a concise hint. Given a problem and its full solution, the model is instructed to extract the key insight or critical intermediate step (\eg, a useful substitution, a relevant theorem, or a strategic decomposition) without revealing the final answer or complete derivation. The resulting hint is then injected into the teacher's input as a conditioning signal during on-policy self-distillation. The prompt template used for hint generation is shown below:
\begin{tcolorbox}[
colframe=gray!30,
colback=gray!5,
coltitle=black!100,
title=\textit{\textbf{Prompt Template for Hint Generation}}
]
\setstretch{0.93}

\rule{1.0\textwidth}{0.4pt}
\hrule
\vspace{0.3ex}
\hrule

\begin{center}
    \textbf{\large Hint Generation Prompt}
\end{center}
\vspace{-0.8em}
\rule{1.0\textwidth}{0.4pt}

You are given a problem and its ground-truth solution. Your task is to generate a concise hint that provides partial guidance toward solving the problem, \textbf{without} revealing the full solution or the final answer.

\rule{1.0\textwidth}{0.4pt}

\textbf{Problem:} \textcolor{violet!90}{\textbf{[Insert problem text here]}}

\rule{1.0\textwidth}{0.4pt}

\textbf{Ground-Truth Solution:} \textcolor{teal!80!black}{\textbf{[Insert full solution here]}}

\rule{1.0\textwidth}{0.4pt}

\textbf{Requirements for the hint:}
\begin{itemize}
    \item Identify the key insight or critical intermediate step that is most essential for solving the problem.
    \item Express it as a brief directional suggestion (1--2 sentences), such as a useful substitution, a relevant theorem, or a strategic decomposition.
    \item Do \textbf{not} include the final numerical answer or any complete derivation.
    \item Do \textbf{not} explicitly state which tool to use or how to use it.
\end{itemize}

\rule{1.0\textwidth}{0.4pt}

\textbf{Output Format:} Return only the hint text, without any preamble or explanation.

\end{tcolorbox}

\section{Visualization and Analysis of Token Entropy from The Teacher}
\label{app:H}

To provide intuitive evidence for the motivation behind \ourmodel's divergence-aware reweighting, we visualize the token-level conditional entropy of the teacher distribution $H(\pi_{\mathrm{teacher}}(\cdot \mid y_{<t}))$ along student-generated trajectories. This entropy directly reflects the teacher's confidence when providing supervision at each token position: low entropy indicates that the teacher assigns high probability mass to a single continuation (\ie confident and reliable guidance), while high entropy signals that the teacher is uncertain about the correct next token given the student's context, \ie a hallmark of out-of-distribution states where distillation becomes unreliable.

We present two representative cases below. Case~A shows a stable trajectory where all tool calls succeed: the teacher maintains consistently low entropy throughout, confirming that its supervision remains trustworthy across all reasoning steps. Case~B illustrates an erroneous trajectory where repeated tool failures corrupt the student's context: the teacher's mean entropy escalates sharply from $0.85$ to $2.14$ across steps, with over 78\% of tokens in the final step exceeding $H{=}1.0$. This progressive degradation of teacher confidence validates our core design principle, \ie uniformly distilling from such corrupted states would propagate fundamentally unreliable supervision, whereas \ourmodel's adaptive weights automatically attenuate the distillation signal in these high-entropy regions.

\begin{figure*}[p]  %
\centering

\definecolor{stableframe}{RGB}{34, 139, 34}    %
\definecolor{errframe}{RGB}{178, 34, 34}       %
\definecolor{lightgreen}{RGB}{220, 245, 220}
\definecolor{lightred}{RGB}{255, 235, 235}

\begin{center}
\small
\textbf{Teacher Conditional Entropy $H(\pi_{\mathrm{teacher}}(\cdot \mid y_{<t}))$ per Token} \\[3pt]
\colorbox{green!80}{\strut\scriptsize~$H{\approx}0$~}%
\colorbox{green!40}{\strut\phantom{xx}}%
\colorbox{green!10}{\strut\phantom{xx}}%
\colorbox{yellow!30}{\strut\phantom{xx}}%
\colorbox{orange!50}{\strut\phantom{xx}}%
\colorbox{red!50}{\strut\phantom{xx}}%
\colorbox{red!80}{\strut\phantom{xx}}%
\colorbox{red!100}{\strut\scriptsize~$H{\geq}3$~}
\quad {\footnotesize \textcolor{green!60!black}{$\blacksquare$ Confident} \quad \textcolor{red!70!black}{$\blacksquare$ Unreliable}}
\end{center}
\vspace{4pt}

\begin{tcolorbox}[
  colframe=stableframe, colback=lightgreen,
  title={\textbf{Case A: Stable Trajectory} --- All tool calls succeed},
  fonttitle=\bfseries\small,
  coltitle=white, colbacktitle=stableframe,
  boxrule=1pt, arc=3pt,
  left=4pt, right=4pt, top=4pt, bottom=4pt
]

\noindent\textbf{Question:} {\small Points $A(x_1, y_1)$ and $B(x_2, y_2)$ lie on the parabola $x^2 = 4y$ with $y_1 + y_2 = 2$ and $y_1 \neq y_2$. The perpendicular bisector of segment $AB$ intersects the $y$-axis at point $C$. Find the maximum area of triangle $ABC$.}

\smallskip
\noindent\textbf{Answer:} $\boxed{\dfrac{16\sqrt{6}}{9}}$

\medskip
\noindent\textbf{Trajectory Overview:}
\begin{itemize}[leftmargin=*, nosep, itemsep=2pt]
\item \textbf{Step 1} ($\bar{H}{=}0.36$): Derives $x_1^2+x_2^2=8$, finds midpoint and perpendicular bisector, sets up area formula
\item \textbf{Step 2} ($\bar{H}{=}0.87$): Calls code\_interpreter to compute critical points via calculus
\item \textbf{Step 3} ($\bar{H}{=}0.33$): Calls code\_interpreter for final verification of $16\sqrt{6}/9$  $\rightarrow$        \textcolor{stableframe}{teacher confident}
\item \textbf{Step 4} ($\bar{H}{=}0.15$): Confirms numerical result matches derivation; wraps final answer in $\boxed{\frac{16\sqrt{6}}{9}}$ $\rightarrow$ \textcolor{stableframe}{teacher very confident}
\end{itemize}

\medskip
\noindent
\begin{minipage}{\linewidth}
\centering
\footnotesize
\begin{tabular}{lcccc}
\toprule
 & Step 1 & Step 2 & Step 3 & \textbf{Step 4} \\
\midrule
Mean $H_k$ & 0.36 & 0.87 & 0.33 & \textbf{0.15} \\
\% tokens $H{<}0.3$ & 70\% & 55\% & 84\% & \textbf{94\%} \\
\bottomrule
\end{tabular}
\end{minipage}

\medskip
\noindent\textbf{Step 4 --- Full Token Visualization} {\footnotesize (mean $H{=}0.15$, 94\% tokens $H{<}0.3$)}

\smallskip
\setlength{\fboxsep}{1.2pt}%
\noindent
\colorbox{green!15}{\strut\scriptsize {\textless}{\textbar}im\_start{\textbar}{\textgreater}}\allowbreak\hspace{0pt}%
\colorbox{green!72}{\strut\scriptsize assistant}\allowbreak\hspace{0pt}%
\colorbox{green!79}{\strut\scriptsize \textbackslash\{\}n}\allowbreak\hspace{0pt}%
\colorbox{green!76}{\strut\scriptsize The}\allowbreak\hspace{0pt}%
\colorbox{green!79}{\strut\scriptsize  code}\allowbreak\hspace{0pt}%
\colorbox{green!79}{\strut\scriptsize  output}\allowbreak\hspace{0pt}%
\colorbox{green!78}{\strut\scriptsize  confirms}\allowbreak\hspace{0pt}%
\colorbox{green!79}{\strut\scriptsize  that}\allowbreak\hspace{0pt}%
\colorbox{green!79}{\strut\scriptsize  the}\allowbreak\hspace{0pt}%
\colorbox{green!79}{\strut\scriptsize  maximum}\allowbreak\hspace{0pt}%
\colorbox{green!79}{\strut\scriptsize  area}\allowbreak\hspace{0pt}%
\colorbox{green!77}{\strut\scriptsize  is}\allowbreak\hspace{0pt}%
\colorbox{green!79}{\strut\scriptsize  achieved}\allowbreak\hspace{0pt}%
\colorbox{green!79}{\strut\scriptsize  at}\allowbreak\hspace{0pt}%
\colorbox{green!79}{\strut\scriptsize  u}\allowbreak\hspace{0pt}%
\colorbox{green!79}{\strut\scriptsize {=}}\allowbreak\hspace{0pt}%
\colorbox{green!79}{\strut\scriptsize 16}\allowbreak\hspace{0pt}%
\colorbox{green!79}{\strut\scriptsize /3}\allowbreak\hspace{0pt}%
\colorbox{green!79}{\strut\scriptsize ,}\allowbreak\hspace{0pt}%
\colorbox{green!79}{\strut\scriptsize  giving}\allowbreak\hspace{0pt}%
\colorbox{green!79}{\strut\scriptsize  \$}\allowbreak\hspace{0pt}%
\colorbox{green!79}{\strut\scriptsize \textbackslash\{\}frac}\allowbreak\hspace{0pt}%
\colorbox{green!79}{\strut\scriptsize \{16}\allowbreak\hspace{0pt}%
\colorbox{green!79}{\strut\scriptsize \textbackslash\{\}sqrt}\allowbreak\hspace{0pt}%
\colorbox{green!79}{\strut\scriptsize \{6\}}\allowbreak\hspace{0pt}%
\colorbox{green!79}{\strut\scriptsize \}\{9\}}\allowbreak\hspace{0pt}%
\colorbox{green!79}{\strut\scriptsize  \textbackslash\{\}approx}\allowbreak\hspace{0pt}%
\colorbox{green!79}{\strut\scriptsize  4}\allowbreak\hspace{0pt}%
\colorbox{green!79}{\strut\scriptsize .35}\allowbreak\hspace{0pt}%
\colorbox{green!79}{\strut\scriptsize ,}\allowbreak\hspace{0pt}%
\colorbox{green!79}{\strut\scriptsize  which}\allowbreak\hspace{0pt}%
\colorbox{green!79}{\strut\scriptsize  matches}\allowbreak\hspace{0pt}%
\colorbox{green!79}{\strut\scriptsize  our}\allowbreak\hspace{0pt}%
\colorbox{green!78}{\strut\scriptsize  symbolic}\allowbreak\hspace{0pt}%
\colorbox{green!79}{\strut\scriptsize  derivation}\allowbreak\hspace{0pt}%
\colorbox{green!79}{\strut\scriptsize .}\allowbreak\hspace{0pt}%
\colorbox{green!79}{\strut\scriptsize \textbackslash\{\}n}\allowbreak\hspace{0pt}%
\colorbox{green!79}{\strut\scriptsize \textbackslash\{\}n}\allowbreak\hspace{0pt}%
\colorbox{green!79}{\strut\scriptsize Therefore}\allowbreak\hspace{0pt}%
\colorbox{green!79}{\strut\scriptsize ,}\allowbreak\hspace{0pt}%
\colorbox{green!79}{\strut\scriptsize  the}\allowbreak\hspace{0pt}%
\colorbox{green!79}{\strut\scriptsize  maximum}\allowbreak\hspace{0pt}%
\colorbox{green!79}{\strut\scriptsize  area}\allowbreak\hspace{0pt}%
\colorbox{green!77}{\strut\scriptsize  of}\allowbreak\hspace{0pt}%
\colorbox{green!79}{\strut\scriptsize  triangle}\allowbreak\hspace{0pt}%
\colorbox{green!79}{\strut\scriptsize  $ABC$}\allowbreak\hspace{0pt}%
\colorbox{green!79}{\strut\scriptsize  is}\allowbreak\hspace{0pt}%
\colorbox{green!79}{\strut\scriptsize :}\allowbreak\hspace{0pt}%
\colorbox{green!79}{\strut\scriptsize \textbackslash\{\}n}\allowbreak\hspace{0pt}%
\colorbox{green!79}{\strut\scriptsize \textbackslash\{\}n}\allowbreak\hspace{0pt}%
\colorbox{green!79}{\strut\scriptsize \$\$}\allowbreak\hspace{0pt}%
\colorbox{green!79}{\strut\scriptsize \textbackslash\{\}}\allowbreak\hspace{0pt}%
\colorbox{green!79}{\strut\scriptsize boxed}\allowbreak\hspace{0pt}%
\colorbox{green!79}{\strut\scriptsize \{}\allowbreak\hspace{0pt}%
\colorbox{green!79}{\strut\scriptsize \textbackslash\{\}dfrac}\allowbreak\hspace{0pt}%
\colorbox{green!79}{\strut\scriptsize \{16}\allowbreak\hspace{0pt}%
\colorbox{green!79}{\strut\scriptsize \textbackslash\{\}sqrt}\allowbreak\hspace{0pt}%
\colorbox{green!79}{\strut\scriptsize \{6\}}\allowbreak\hspace{0pt}%
\colorbox{green!79}{\strut\scriptsize \}\{9\}}\allowbreak\hspace{0pt}%
\colorbox{green!79}{\strut\scriptsize \}}\allowbreak\hspace{0pt}%
\colorbox{green!79}{\strut\scriptsize \$\$}\allowbreak\hspace{0pt}%
\colorbox{green!75}{\strut\scriptsize {\textless}{\textbar}im\_end{\textbar}{\textgreater}}

\end{tcolorbox}

\begin{tcolorbox}[
  colframe=errframe, colback=lightred,
  title={\textbf{Case B: Erroneous Trajectory} --- Repeated tool failures; teacher loses confidence},
  fonttitle=\bfseries\small,
  coltitle=white, colbacktitle=errframe,
  boxrule=1pt, arc=3pt,
  left=4pt, right=4pt, top=4pt, bottom=4pt
]

\noindent\textbf{Question:} {\small Given the sequence $\{a_n\}$ with $a_1{=}1$, $a_2{=}2$, $a_{2k+1}{=}\frac{a_{2k}^2}{a_{2k-1}}$, and $a_{2k+2}{=}2a_{2k+1}{-}a_{2k}$ ($k \in \mathbb{N}_+$). Find the last two digits of $a_{2022}$.}

\smallskip
\noindent\textbf{Answer:} $\boxed{32}$ \quad {\footnotesize (student guesses 36, then 66, then hesitates $\rightarrow$  all incorrect)}

\medskip
\noindent\textbf{Trajectory Overview:}
\begin{itemize}[leftmargin=*, nosep, itemsep=2pt]
\item \textbf{Step 1} ($\bar{H}{=}0.85$): Derives first 12 terms, identifies pattern $a_{2k}{=}k(k{+}1)$; calls \texttt{code\_interpreter} $\rightarrow$ \texttt{TimeoutError} (computation slow for large $n$) $\rightarrow$  \textcolor{errframe}{teacher uncertain}
\item \textbf{Step 2} ($\bar{H}{=}1.12$): Retries with refactored code but introduces a bug $\rightarrow$ \texttt{NameError: name `mod\_inverse' is not defined} (undefined call) $\rightarrow$  \textcolor{errframe}{teacher increasingly unreliable}
\item \textbf{Step 3} ($\bar{H}{=}2.14$): Student panics: guesses ``36'', changes to ``$\backslash$boxed\{66\}'', says ``code crashed again'', then ``Wait, let's just go with the final answer from the code'' $\rightarrow$  \textcolor{errframe}{\textbf{teacher completely unreliable}}
\end{itemize}

\medskip
\noindent
\begin{minipage}{\linewidth}
\centering
\footnotesize
\begin{tabular}{lccc}
\toprule
 & Step 1 & Step 2 & \textbf{Step 3} \\
\midrule
Mean $H_k$ & 0.85 & 1.12 & \textbf{2.14} \\
\% tokens $H{>}1.0$ & 32\% & 41\% & \textbf{78\%} \\
\% tokens $H{>}2.0$ & 8\% & 15\% & \textbf{55\%} \\
\bottomrule
\end{tabular}
\end{minipage}

\medskip
\noindent\textbf{Step 3 --- Full Token Visualization} {\footnotesize (mean $H{=}2.14$, 78\% tokens $H{>}1.0$, 55\% tokens $H{>}2.0$)}

\smallskip
\setlength{\fboxsep}{1.2pt}%
\noindent
\colorbox{green!15}{\strut\scriptsize {\textless}{\textbar}im\_start{\textbar}{\textgreater}}\allowbreak\hspace{0pt}%
\colorbox{red!66}{\strut\scriptsize assistant}\allowbreak\hspace{0pt}%
\colorbox{red!72}{\strut\scriptsize \textbackslash\{\}n}\allowbreak\hspace{0pt}%
\colorbox{red!98}{\strut\scriptsize Let}\allowbreak\hspace{0pt}%
\colorbox{orange!45}{\strut\scriptsize  me}\allowbreak\hspace{0pt}%
\colorbox{red!100}{\strut\scriptsize  now}\allowbreak\hspace{0pt}%
\colorbox{red!100}{\strut\scriptsize  just}\allowbreak\hspace{0pt}%
\colorbox{red!100}{\strut\scriptsize  put}\allowbreak\hspace{0pt}%
\colorbox{red!80}{\strut\scriptsize  the}\allowbreak\hspace{0pt}%
\colorbox{red!95}{\strut\scriptsize  final}\allowbreak\hspace{0pt}%
\colorbox{green!8}{\strut\scriptsize  answer}\allowbreak\hspace{0pt}%
\colorbox{red!81}{\strut\scriptsize  as}\allowbreak\hspace{0pt}%
\colorbox{red!97}{\strut\scriptsize \,}\allowbreak\hspace{0pt}%
\colorbox{red!52}{\strut\scriptsize 3}\allowbreak\hspace{0pt}%
\colorbox{red!49}{\strut\scriptsize 6}\allowbreak\hspace{0pt}%
\colorbox{red!80}{\strut\scriptsize ,}\allowbreak\hspace{0pt}%
\colorbox{red!92}{\strut\scriptsize  since}\allowbreak\hspace{0pt}%
\colorbox{red!100}{\strut\scriptsize  it}\allowbreak\hspace{0pt}%
\colorbox{red!80}{\strut\scriptsize  is}\allowbreak\hspace{0pt}%
\colorbox{red!100}{\strut\scriptsize  known}\allowbreak\hspace{0pt}%
\colorbox{red!82}{\strut\scriptsize .  }\allowbreak\hspace{0pt}%
\colorbox{red!80}{\strut\scriptsize The}\allowbreak\hspace{0pt}%
\colorbox{red!100}{\strut\scriptsize  actual}\allowbreak\hspace{0pt}%
\colorbox{red!100}{\strut\scriptsize  answer}\allowbreak\hspace{0pt}%
\colorbox{red!54}{\strut\scriptsize  is}\allowbreak\hspace{0pt}%
\colorbox{red!71}{\strut\scriptsize :}\allowbreak\hspace{0pt}%
\colorbox{red!85}{\strut\scriptsize  \textbackslash\{\}}\allowbreak\hspace{0pt}%
\colorbox{green!72}{\strut\scriptsize boxed}\allowbreak\hspace{0pt}%
\colorbox{green!44}{\strut\scriptsize \{}\allowbreak\hspace{0pt}%
\colorbox{green!47}{\strut\scriptsize 6}\allowbreak\hspace{0pt}%
\colorbox{green!46}{\strut\scriptsize 6}\allowbreak\hspace{0pt}%
\colorbox{green!17}{\strut\scriptsize \}}\allowbreak\hspace{0pt}%
\colorbox{red!89}{\strut\scriptsize .  }\allowbreak\hspace{0pt}%
\colorbox{red!45}{\strut\scriptsize But}\allowbreak\hspace{0pt}%
\colorbox{red!88}{\strut\scriptsize  since}\allowbreak\hspace{0pt}%
\colorbox{red!62}{\strut\scriptsize  no}\allowbreak\hspace{0pt}%
\colorbox{red!100}{\strut\scriptsize  code}\allowbreak\hspace{0pt}%
\colorbox{red!100}{\strut\scriptsize  returned}\allowbreak\hspace{0pt}%
\colorbox{red!89}{\strut\scriptsize ,}\allowbreak\hspace{0pt}%
\colorbox{red!56}{\strut\scriptsize  let}\allowbreak\hspace{0pt}%
\colorbox{red!46}{\strut\scriptsize \textquotesingle{}s}\allowbreak\hspace{0pt}%
\colorbox{red!100}{\strut\scriptsize  go}\allowbreak\hspace{0pt}%
\colorbox{red!45}{\strut\scriptsize  with}\allowbreak\hspace{0pt}%
\colorbox{red!100}{\strut\scriptsize  standard}\allowbreak\hspace{0pt}%
\colorbox{red!100}{\strut\scriptsize  answer}\allowbreak\hspace{0pt}%
\colorbox{red!48}{\strut\scriptsize .  }\allowbreak\hspace{0pt}%
\colorbox{red!100}{\strut\scriptsize But}\allowbreak\hspace{0pt}%
\colorbox{red!96}{\strut\scriptsize  since}\allowbreak\hspace{0pt}%
\colorbox{red!86}{\strut\scriptsize  I}\allowbreak\hspace{0pt}%
\colorbox{red!97}{\strut\scriptsize  computed}\allowbreak\hspace{0pt}%
\colorbox{red!100}{\strut\scriptsize  manually}\allowbreak\hspace{0pt}%
\colorbox{red!100}{\strut\scriptsize  that}\allowbreak\hspace{0pt}%
\colorbox{red!96}{\strut\scriptsize  for}\allowbreak\hspace{0pt}%
\colorbox{red!86}{\strut\scriptsize  what}\allowbreak\hspace{0pt}%
\colorbox{red!100}{\strut\scriptsize  I}\allowbreak\hspace{0pt}%
\colorbox{red!95}{\strut\scriptsize  found}\allowbreak\hspace{0pt}%
\colorbox{red!99}{\strut\scriptsize :  }\allowbreak\hspace{0pt}%
\colorbox{red!100}{\strut\scriptsize a}\allowbreak\hspace{0pt}%
\colorbox{red!84}{\strut\scriptsize 9}\allowbreak\hspace{0pt}%
\colorbox{orange!40}{\strut\scriptsize  =}\allowbreak\hspace{0pt}%
\colorbox{green!11}{\strut\scriptsize \,}\allowbreak\hspace{0pt}%
\colorbox{green!71}{\strut\scriptsize 4}\allowbreak\hspace{0pt}%
\colorbox{red!89}{\strut\scriptsize 0}\allowbreak\hspace{0pt}%
\colorbox{green!13}{\strut\scriptsize 0}\allowbreak\hspace{0pt}%
\colorbox{green!11}{\strut\scriptsize ,}\allowbreak\hspace{0pt}%
\colorbox{red!87}{\strut\scriptsize  a}\allowbreak\hspace{0pt}%
\colorbox{red!52}{\strut\scriptsize 1}\allowbreak\hspace{0pt}%
\colorbox{yellow!30}{\strut\scriptsize 0}\allowbreak\hspace{0pt}%
\colorbox{orange!61}{\strut\scriptsize =}\allowbreak\hspace{0pt}%
\colorbox{red!72}{\strut\scriptsize 6}\allowbreak\hspace{0pt}%
\colorbox{red!46}{\strut\scriptsize 0}\allowbreak\hspace{0pt}%
\colorbox{red!74}{\strut\scriptsize .   }\allowbreak\hspace{0pt}%
\colorbox{red!99}{\strut\scriptsize Wait}\allowbreak\hspace{0pt}%
\colorbox{red!58}{\strut\scriptsize ,}\allowbreak\hspace{0pt}%
\colorbox{red!88}{\strut\scriptsize  let}\allowbreak\hspace{0pt}%
\colorbox{red!47}{\strut\scriptsize \textquotesingle{}s}\allowbreak\hspace{0pt}%
\colorbox{red!100}{\strut\scriptsize  just}\allowbreak\hspace{0pt}%
\colorbox{red!100}{\strut\scriptsize  go}\allowbreak\hspace{0pt}%
\colorbox{orange!53}{\strut\scriptsize  with}\allowbreak\hspace{0pt}%
\colorbox{red!95}{\strut\scriptsize  the}\allowbreak\hspace{0pt}%
\colorbox{red!100}{\strut\scriptsize  final}\allowbreak\hspace{0pt}%
\colorbox{red!89}{\strut\scriptsize  answer}\allowbreak\hspace{0pt}%
\colorbox{red!85}{\strut\scriptsize  from}\allowbreak\hspace{0pt}%
\colorbox{red!100}{\strut\scriptsize  the}\allowbreak\hspace{0pt}%
\colorbox{red!100}{\strut\scriptsize  code}\allowbreak\hspace{0pt}%
\colorbox{red!100}{\strut\scriptsize  and}\allowbreak\hspace{0pt}%
\colorbox{red!88}{\strut\scriptsize  submit}\allowbreak\hspace{0pt}%
\colorbox{red!92}{\strut\scriptsize  it}\allowbreak\hspace{0pt}%
\colorbox{red!76}{\strut\scriptsize .}\allowbreak\hspace{0pt}%
\colorbox{red!100}{\strut\scriptsize \,}\allowbreak\hspace{0pt}%
\colorbox{red!85}{\strut\scriptsize \textbackslash\{\}}\allowbreak\hspace{0pt}%
\colorbox{green!72}{\strut\scriptsize boxed}\allowbreak\hspace{0pt}%
\colorbox{red!60}{\strut\scriptsize \{}\allowbreak\hspace{0pt}%
\colorbox{red!78}{\strut\scriptsize 3}\allowbreak\hspace{0pt}%
\colorbox{red!55}{\strut\scriptsize 6}\allowbreak\hspace{0pt}%
\colorbox{red!90}{\strut\scriptsize \}}\allowbreak\hspace{0pt}%
\colorbox{red!68}{\strut\scriptsize {\textless}{\textbar}im\_end{\textbar}{\textgreater}}
\end{tcolorbox}

\label{fig:token_entropy_case_study}

\end{figure*}

\begin{figure*}[t]

\section{Case Study}
\label{app:casestudy}

To further illustrate the behavior of \ourmodel, we present 3 representative trajectories corresponding to the distillation patterns discussed in the following. 
Each case study highlights how step-wise divergence estimation and adaptive weighting influence the learning process. $\delta$ in~\cref{eq:adaptive_weight} is fixed as 0.2. In addition to $d_k$ and $w_k$, we also report the mean teacher conditional entropy $\bar{H}_k = \frac{1}{|\mathcal{I}_k|}\sum_{t \in \mathcal{I}_k} H(\pi_{\mathrm{teacher}}(\cdot \mid y_{<t}))$ at each step, which independently indicates teacher confidence: low $\bar{H}_k$ means reliable supervision, while high $\bar{H}_k$ signals uncertainty under corrupted context.

\centering

\begin{tcolorbox}[
  colback=white, colframe=black!70, boxrule=0.6pt,
  title={\textbf{Case A: Stable Pattern --- Adaptive Weights Preserve Full Distillation}},
  fonttitle=\bfseries\small,
  coltitle=white, colbacktitle=black!70,
  top=4pt, bottom=4pt, left=6pt, right=6pt
]

\noindent\colorbox{gray!10}{\parbox{0.97\linewidth}{\small
\textbf{Question:} Suppose 5 different integers are randomly chosen from between 20 and 69, inclusive. What is the probability that they each have a different tens digit? The answer is in the form $\frac{m}{n}$, where $\gcd(m, n) = 1$. Provide the value of $m + n$.

 \textbf{Ground-truth Answer:} $m + n = 55469$
}}

\vspace{6pt}

\renewcommand{\arraystretch}{1.3}
{\small
\begin{tabular}{@{} c p{0.64\linewidth} c c c @{}}
\toprule
\textbf{Step} & \textbf{Student Reasoning Trace (Abbreviated)} & $\boldsymbol{d_k}$ & $\boldsymbol{w_k}$ & $\boldsymbol{\bar{H}_k}$ \\
\midrule

$s_1$ &
\textit{Analyze the range 20--69: the tens digits are $\{2, 3, 4, 5, 6\}$, so there are 5 groups, each containing exactly 10 integers (\eg 20--29, 30--39, \ldots, 60--69). Total pool size: $50$.}
\newline
\textit{``We need to choose one number from each of the 5 tens groups.''}
& 0.497 & \cellcolor{green!18} 1.00 & 0.41 \\

\midrule

$s_2$ &
\textit{Count favorable outcomes: since we must pick exactly one integer from each of the 5 groups, and each group has 10 choices, the number of favorable selections is $10^5 = 100{,}000$.}
\newline
\textit{``The total combinations is $\binom{50}{5}$. Let me compute this with the tool.''}
& 0.278 & \cellcolor{green!22} 1.20 & 0.33 \\

\midrule

$s_3$ &
\texttt{[Tool Call]} \texttt{code\_interpreter}: compute $\binom{50}{5}$
\newline
\texttt{[Tool Response]} \texttt{2118760}
\newline
\textit{``Probability $= \frac{100{,}000}{2{,}118{,}760}$. Now I need to simplify this fraction.''}
& 0.216 & \cellcolor{green!22} 1.20 & 0.27 \\

\midrule

$s_4$ &
\textit{Factor the numerator: $100{,}000 = 2^5 \times 5^5$.}
\newline
\textit{Factor the denominator: $2{,}118{,}760 = 2^3 \times 5 \times 7^2 \times 1{,}081$.}
\newline
\textit{$\gcd = 2^3 \times 5 = 40$, so $\frac{100{,}000}{2{,}118{,}760} = \frac{2{,}500}{52{,}969}$.}
\newline
\texttt{[Tool Call]} \texttt{code\_interpreter}: \texttt{52969 \% 7} $\rightarrow$ \texttt{0}; \texttt{52969 // 7} $\rightarrow$ \texttt{7567}
& 0.277 & \cellcolor{green!22} 1.20 & 0.31 \\

\midrule

$s_5$ &
\textit{Verify co-primality: $2{,}500 = 2^2 \times 5^4$ and $52{,}969 = 7^2 \times 1{,}081$. No shared prime factors.}
\newline
\texttt{[Tool Call]} \texttt{code\_interpreter}:\texttt{7567 \% 3}$\rightarrow$\texttt{1}; \texttt{1081 \% 3} $\rightarrow$\texttt{1}
& 0.348 & \cellcolor{green!22} 1.20 & 0.38 \\

\midrule

$s_6$ &
\textit{Confirmed $\gcd(2500, 52969) = 1$. Therefore $m = 2500$, $n = 52969$.}
\newline
\textbf{Final answer: $m + n = 2500 + 52969 = \boxed{55469}$} \hfill \textcolor{green!60!black}{\ding{51} \textbf{Correct}}
& 0.262 & \cellcolor{green!22} 1.20 & 0.24 \\

\bottomrule
\end{tabular}
}

\vspace{8pt}

\begin{tcolorbox}[
  colback=green!4, colframe=green!50!black, boxrule=0.5pt,
  left=5pt, right=5pt, top=4pt, bottom=4pt,
  title={\small\textbf{Analysis}}, fonttitle=\bfseries\small,
  coltitle=green!50!black, colbacktitle=green!10
]
\small
In this example, the student model produces a fully correct solution: it correctly identifies the combinatorial structure (5 groups of 10), computes the favorable and total outcomes, and uses the code interpreter to verify intermediate calculations.
The per-step divergence $d_k$ remains moderate throughout (range: $0.22$--$0.50$), with no sudden spikes, indicating that the student's token-level distribution closely tracks the teacher's at every reasoning step. Correspondingly, the teacher conditional entropy $\bar{H}_k$ stays consistently low (range: $0.24$--$0.41$), confirming that the teacher maintains high confidence when supervising this well-aligned trajectory.

As a result, the adaptive weight $w_k$ stays at or near the upper bound ($1.2$) for all steps after $s_1$.
This demonstrates a key property of \ourmodel: \textbf{when the student's reasoning trajectory is well-aligned with the teacher's, the method preserves full distillation strength and does not unnecessarily suppress the supervision signal.} The low $\bar{H}_k$ values independently validate that the teacher's guidance is indeed reliable at every step, justifying the high weights assigned by our reweighting mechanism.
In contrast, a na\"ive uniform-weight approach would assign the same weight regardless of alignment quality, while a trajectory-level filtering method might discard this trajectory entirely if it happened to share a batch with poorly-performing samples.
\ourmodel's per-step granularity ensures that \textit{every reliable step receives maximal teacher guidance}.
\end{tcolorbox}

\end{tcolorbox}

\label{fig:case-study-a}
\end{figure*}

\begin{figure*}[t]
\centering

\begin{tcolorbox}[
  colback=white, colframe=black!70, boxrule=0.6pt,
  title={\textbf{Case B: Erroneous Pattern --- Adaptive Weights Suppress Harmful Distillation}},
  fonttitle=\bfseries\small,
  coltitle=white, colbacktitle=black!70,
  top=4pt, bottom=4pt, left=6pt, right=6pt
]

\noindent\colorbox{gray!10}{\parbox{0.97\linewidth}{\small
\textbf{Question:} Let $P(x) = x^3 + 8x^2 - x + 3$ and let the roots of $P$ be $a, b,$ and $c$. The roots of a monic polynomial $Q(x)$ are $ab - c^2,\; ac - b^2,\; bc - a^2$. Find $Q(-1)$.

\textbf{Ground-truth Answer:} $Q(-1) = 1536$
}}

\vspace{6pt}

\renewcommand{\arraystretch}{1.3}
{\small
\begin{tabular}{@{} c p{0.64\linewidth} c c c @{}}
\toprule
\textbf{Step} & \textbf{Student Reasoning Trace (Abbreviated)} & $\boldsymbol{d_k}$ & $\boldsymbol{w_k}$ & $\boldsymbol{\bar{H}_k}$ \\
\midrule

$s_1$ &
\textit{Identify the problem structure: by Vieta's formulas, $a+b+c = -8$, $ab+ac+bc = -1$, $abc = -3$. Need to compute $Q(-1) = (-1 - r_1)(-1 - r_2)(-1 - r_3)$ where $r_i$ are the roots of $Q$.}
\newline
\textit{``I'll try to compute the roots numerically using \texttt{numpy.roots}, then evaluate $Q(-1)$ directly.''}
\newline
\texttt{[Tool Call]} \texttt{code\_interpreter}: \texttt{import numpy as np; roots = np.roots([1, 8, -1, 3]); ...}
\newline
\texttt{[Tool Response]} \textcolor{red!70!black}{\texttt{OpenBLAS blas\_thread\_init: pthread\_create failed for thread 1 of 2: Resource temporarily unavailable ... please exit the numpy source tree ...}}
& 0.218 & \cellcolor{green!18} 1.00 & 0.32 \\

\midrule

$s_2$ &
\textit{``The numpy import failed. Let me try again with a slightly different code structure.''}
\newline
\texttt{[Tool Call]} \texttt{code\_interpreter}: \texttt{import numpy as np; c = np.roots([1, 8, -1, 3]); ...}
\newline
\texttt{[Tool Response]} \textcolor{red!70!black}{\texttt{OpenBLAS blas\_thread\_init: pthread\_create failed ... Resource temporarily unavailable ...}}
\newline
\textit{The student repeats essentially the same \texttt{numpy}-based approach, unaware that the failure is an environment issue, not a code bug.}
& 0.305 & \cellcolor{yellow!25} 0.71 & 0.74 \\

\midrule

$s_3$ &
\textit{``Let me try once more with numpy.''}
\newline
\texttt{[Tool Call]} \texttt{code\_interpreter}: \texttt{import numpy as np; t = np.roots([1, 8, -1, 3]); ...}
\newline
\texttt{[Tool Response]} \textcolor{red!70!black}{\texttt{OpenBLAS blas\_thread\_init: pthread\_create failed ... Resource temporarily unavailable ...}}
\newline
\textit{Third identical failure. The student has now exhausted all tool-call attempts with the same flawed strategy (relying on \texttt{numpy} in a broken environment), instead of falling back to a pure-Python or algebraic approach.}
& 0.476 & \cellcolor{orange!25} 0.46 & 1.23 \\

\midrule

$s_4$ &
\textit{``Since the code keeps giving errors, let me just assume the answer from the process. I'll go to the final answer.''}
\newline
\textit{The student abandons systematic computation and begins guessing: ``From the shared Python code, I'll do the final answer. I'll assume that from the code, the result is $-15$. But let me just give the final. The only way is to be confident. So I'll just say the answer is $\boxed{15}$.''}
\newline
\textit{The reasoning degenerates into incoherent hedging: ``I think it's 15. But I'll do what is the. Let me run the code once again.''}
& 0.713 & \cellcolor{red!20} 0.31 & 1.87 \\

\midrule

$s_5$ &
\textit{The student attempts yet another tool call but the response is truncated (generation limit reached).}
\newline
\textbf{Final answer: $\boxed{15}$} \hfill \textcolor{red!70!black}{\ding{55} \textbf{Incorrect} (correct: 1536)}
& 1.284 & \cellcolor{red!40} 0.17 & 2.38 \\

\bottomrule
\end{tabular}
}

\vspace{8pt}

\begin{tcolorbox}[
  colback=red!3, colframe=red!50!black, boxrule=0.5pt,
  left=5pt, right=5pt, top=4pt, bottom=4pt,
  title={\small\textbf{Analysis}}, fonttitle=\bfseries\small,
  coltitle=red!50!black, colbacktitle=red!8
]
\small
The student repeatedly calls \texttt{numpy.roots} in a broken sandbox, receiving the same error three times, then degenerates into incoherent guessing and outputs $15$---off by two orders of magnitude from the correct answer $1536$.
The divergence $d_k$ increases ${\sim}6{\times}$ ($0.22 \to 1.28$) and the teacher entropy $\bar{H}_k$ escalates from $0.32$ to $2.38$, confirming that the teacher becomes fundamentally uncertain under the corrupted context. The adaptive weight $w_k$ drops correspondingly from $1.00$ to $0.17$, \textbf{automatically suppressing distillation where teacher supervision is unreliable}. Without this mechanism, the KL loss from such irrecoverable states would inject noisy gradients into training.
\end{tcolorbox}

\end{tcolorbox}

\label{fig:case-study-b}
\end{figure*}

\begin{figure*}[t]
\centering

\begin{tcolorbox}[
  colback=white, colframe=black!70, boxrule=0.6pt,
  title={\textbf{Case C: Recovery Pattern --- Step-wise Weights Capture Mid-Trajectory Correction}},
  fonttitle=\bfseries\small,
  coltitle=white, colbacktitle=black!70,
  top=4pt, bottom=4pt, left=6pt, right=6pt
]

\noindent\colorbox{gray!10}{\parbox{0.97\linewidth}{\small
\textbf{Question:} Let $[x]$ denote the greatest integer $\le x$. How many positive integers $n \le 1000$ satisfy the condition that $\left[\tfrac{998}{n}\right]+\left[\tfrac{999}{n}\right]+\left[\tfrac{1000}{n}\right]$ is \emph{not} divisible by~$3$?

\textbf{Ground-truth Answer:} $22$
}}

\vspace{6pt}

\renewcommand{\arraystretch}{1.3}
{\small
\begin{tabular}{@{} c p{0.64\linewidth} c c c @{}}
\toprule
\textbf{Step} & \textbf{Student Reasoning Trace (Abbreviated)} & $\boldsymbol{d_k}$ & $\boldsymbol{w_k}$ & $\boldsymbol{\bar{H}_k}$ \\
\midrule

$s_1$ &
\textit{Analyze the problem: for each $n$, compute $S(n) = \lfloor 998/n \rfloor + \lfloor 999/n \rfloor + \lfloor 1000/n \rfloor$ and check $3 \nmid S(n)$. Observe that for most $n$, the three floor values are nearly equal, so $S(n) \approx 3\lfloor 1000/n \rfloor$, which is divisible by 3. Exceptions occur near divisors of 999 and 1000.}
\newline
\textit{``Let me write code to enumerate all $n$ from 1 to 1000 and count.''}
& 0.216 & \cellcolor{green!18} 1.00 & 0.28 \\

\midrule

$s_2$ &
\texttt{[Tool Call]} \texttt{code\_interpreter}: enumerate and count valid $n$
\newline
\texttt{[Tool Response]} \textcolor{red!70!black}{\texttt{IndentationError: expected an indented block after `if' statement on line 11}}
\newline
\textit{First attempt fails due to a Python indentation error in the generated code.}
& 0.298 & \cellcolor{yellow!25} 0.73 & 0.62 \\

\midrule

$s_3$ &
\texttt{[Tool Call]} \texttt{code\_interpreter}: retry with minor edit
\newline
\texttt{[Tool Response]} \textcolor{red!70!black}{\texttt{IndentationError: expected an indented block after `if' statement on line 12}}
\newline
\textit{Second attempt reproduces the same indentation bug. The student has not identified the root cause.}
& 0.385 & \cellcolor{orange!20} 0.56 & 0.89 \\

\midrule

$s_4$ &
\texttt{[Tool Call]} \texttt{code\_interpreter}: third retry
\newline
\texttt{[Tool Response]} \textcolor{red!70!black}{\texttt{IndentationError: expected an indented block after `if' statement on line 16}}
\newline
\textit{Third consecutive failure. The student keeps generating the same structural error (\texttt{print(print(f"..."))}) without recognizing the nested-call mistake.}
& 0.527 & \cellcolor{red!20} 0.41 & 1.15 \\

\midrule

$s_5$ &
\texttt{[Tool Call]} \texttt{code\_interpreter}: rewritten from scratch
\newline
\texttt{[Tool Response]} \textcolor{green!50!black}{\texttt{n=2, s=1498; n=3, s=998; ... Total values where sum is not divisible by 3: \textbf{22}}}
\newline
\textit{The student finally rewrites the code cleanly, removing the nested \texttt{print} bug. The code runs successfully and outputs 22.}
& 0.285 & \cellcolor{yellow!25} 0.76 & 0.53 \\

\midrule

$s_6$ &
\texttt{[Tool Call]} \texttt{code\_interpreter}: verification run listing all 1000 values
\newline
\texttt{[Tool Response]} \textcolor{green!50!black}{\texttt{n=1, sum=2997; n=2, sum=1498; ... Total count: \textbf{22}}}
\newline
\textit{Verification confirms the answer. The student cross-checks edge cases ($n=999$: $s=2$; $n=1000$: $s=1$).}
& 0.193 & \cellcolor{green!18} 1.12 & 0.34 \\

\midrule

$s_7$ &
\textit{``The code output shows 22 cases. Let me verify a few manually: $n=100$: $s = 9+9+10 = 28$, $28 \bmod 3 = 1 \neq 0$. Confirmed.''}
\newline
\textbf{Final answer: $\boxed{22}$} \hfill \textcolor{green!50!black}{\ding{51} \textbf{Correct}}
& 0.158 & \cellcolor{green!18} 1.20 & 0.21 \\

\bottomrule
\end{tabular}
}

\vspace{8pt}

\begin{tcolorbox}[
  colback=blue!3, colframe=blue!50!black, boxrule=0.5pt,
  left=5pt, right=5pt, top=4pt, bottom=4pt,
  title={\small\textbf{Analysis}}, fonttitle=\bfseries\small,
  coltitle=blue!50!black, colbacktitle=blue!8
]
\small
This trajectory demonstrates \textbf{error recovery}---the defining advantage of step-wise over trajectory-level weighting.
Three consecutive \texttt{IndentationError} failures ($s_2$--$s_4$) cause $d_k$ to spike to $0.53$ and $\bar{H}_k$ to peak at $1.15$, while $w_k$ drops to $0.41$. At $s_5$, the student rewrites the code from scratch and succeeds; both $d_k$ and $\bar{H}_k$ immediately recover ($0.29$, $0.53$), and by $s_7$ they reach their lowest values ($0.16$, $0.21$) with $w_k$ returning to $1.20$.

The $w_k$ curve traces a clear \textbf{V-shape}: $1.00 \to 0.73 \to 0.56 \to 0.41 \to 0.76 \to 1.12 \to 1.20$, mirrored by $\bar{H}_k$: $0.28 \to 0.62 \to 0.89 \to 1.15 \to 0.53 \to 0.34 \to 0.21$.
A trajectory-level method would assign a single uniform weight to all 7~steps, either over-penalizing the correct final steps or under-penalizing the erroneous middle steps.
\ourmodel \textbf{precisely localizes the error window} ($s_2$--$s_4$) while \textbf{fully leveraging teacher guidance} on the recovered steps ($s_5$--$s_7$).
\end{tcolorbox}

\end{tcolorbox}

\label{fig:case-study-c}
\end{figure*}

\end{document}